%% file: main.tex
\documentclass[11pt]{article}

\usepackage[letterpaper,margin=1in]{geometry}
\usepackage[T1]{fontenc}
\usepackage{lmodern}
\usepackage{amsmath,amssymb,amsthm,mathtools,bm}
\usepackage{booktabs}
\usepackage{graphicx}
\usepackage{subcaption}
\usepackage{algorithm}
\usepackage{algorithmic}
\usepackage{array}
\usepackage{multirow}
\usepackage{xcolor}
\usepackage{microtype}
\usepackage[numbers,sort&compress]{natbib}
\usepackage[hidelinks]{hyperref}
\usepackage[nameinlink,capitalize,noabbrev]{cleveref}
\usepackage{url}
\hypersetup{
  pdftitle={Matched Queries for Curvature and Density at Branching Junctions},
  pdfauthor={Ziqi Zhao and Qingjian Ni}
}

\input{macros}
\input{tables/results_macros}
\input{tables/general_end_to_end_macros}

\title{Matched Queries for Curvature and Density at Branching Junctions}
\author{
  Ziqi Zhao\\
  School of Computer Science and Engineering\\
  Southeast University\\
  Nanjing, China\\
  \texttt{ziqizhao@seu.edu.cn}
  \and
  Qingjian Ni\thanks{Corresponding author.}\\
  School of Computer Science and Engineering\\
  Southeast University\\
  Nanjing, China\\
  \texttt{nqj@seu.edu.cn}
}
\date{}

\begin{document}
\maketitle

\input{sections/00_abstract}
\input{sections/01_introduction}
\input{sections/02_model_expansion}
\input{sections/03_identifiability_queries}
\input{sections/04_center_stability}
\input{sections/05_experiments}
\input{sections/06_related_limitations}

\bibliographystyle{plainnat}
\bibliography{refs}

\clearpage
\appendix
\input{appendices/A_expansion}
\input{appendices/B_identifiability}
\input{appendices/C_query_complexity}
\input{appendices/D_center}
\input{appendices/E_stability_kde}
\input{appendices/F_experimental_details}

\end{document}

%% file: macros.tex
\newtheorem{theorem}{Theorem}[section]
\newtheorem{proposition}[theorem]{Proposition}
\newtheorem{lemma}[theorem]{Lemma}
\newtheorem{corollary}[theorem]{Corollary}
\newtheorem{assumption}[theorem]{Assumption}
\theoremstyle{definition}
\newtheorem{definition}[theorem]{Definition}
\theoremstyle{remark}

\newcommand{\R}{\mathbb{R}}
\newcommand{\Sph}{\mathbb{S}}
\newcommand{\E}{\mathbb{E}}

\newcommand{\cD}{\mathcal{D}}
\newcommand{\cJ}{\mathcal{J}}

\newcommand{\cQ}{\mathcal{Q}}

\newcommand{\ip}[2]{\left\langle #1,#2\right\rangle}
\newcommand{\norm}[1]{\left\lVert #1\right\rVert}
\newcommand{\abs}[1]{\left\lvert #1\right\rvert}
\newcommand{\op}{\mathrm{op}}
\newcommand{\supp}{\mathrm{supp}}

\newcommand{\Span}{\mathrm{span}}
\newcommand{\Id}{I}
\newcommand{\dd}{\,\mathrm{d}}

\newcommand{\trans}{\mathsf{T}}

\DeclareMathOperator{\Cov}{Cov}
\DeclareMathOperator{\ess}{ESS}

%% file: tables/results_macros.tex
\newcommand{\ResultLabel}{paper}

\newcommand{\PopulationMedianSlope}{0.998}
\newcommand{\PopulationSlopeQfive}{0.957}
\newcommand{\PopulationSlopeQninetyfive}{1.03}
\newcommand{\PopulationMedianRtwo}{1}
\newcommand{\PopulationInstances}{180}

\newcommand{\QueryDOptMedianRelativeError}{0.00552}

\newcommand{\QueryRandomDOptErrorRatio}{2.89}

\newcommand{\KDEEndpointReductionMinPercent}{52.1}
\newcommand{\KDEEndpointReductionMaxPercent}{66}

\newcommand{\KDEWithinPointFourAgreementPercent}{68.8}
\newcommand{\EndToEndKDEBlindCountPercent}{98.8}
\newcommand{\EndToEndPopulationBlindCountPercent}{100}

\newcommand{\LearnedJobCount}{20}
\newcommand{\LearnedBestMedianStep}{5000}
\newcommand{\LearnedModelParameters}{135042}
\newcommand{\LearnedTrainingSampleMin}{7544}
\newcommand{\LearnedTrainingSampleMax}{8110}

\newcommand{\LearnedGPUJobHours}{1.33}
\newcommand{\LearnedMedianJobSeconds}{241}
\newcommand{\LearnedElapsedSeconds}{\ensuremath{2.47\times 10^{3}}}
\newcommand{\StabilityNearCollisionError}{0.698}
\newcommand{\StabilityWellSeparatedError}{0.00379}
\newcommand{\StabilityErrorRatio}{184}
\newcommand{\StabilityNearCollisionCondition}{\ensuremath{5.61\times 10^{4}}}
\newcommand{\StabilityWellSeparatedCondition}{24.6}

\newcommand{\ResponsePopulationMildRatio}{7.75}
\newcommand{\ResponsePopulationModerateRatio}{26.6}
\newcommand{\ResponsePopulationStrongRatio}{49.4}
\newcommand{\ResponseKDEStrongRatio}{1.96}
\newcommand{\ResponsePopulationSmallSigmaStrongRatio}{143}
\newcommand{\ScaleLargestDimension}{20}
\newcommand{\ScaleLargestBranches}{16}
\newcommand{\ScaleLargestParameters}{340}
\newcommand{\ScaleLargestQueries}{680}

\newcommand{\ScaleLargestExactError}{0.0825}
\newcommand{\ScaleLargestStrongError}{85.6}

\newcommand{\DiagnosticQueryRowRatio}{1}
\newcommand{\DiagnosticPopulationEstimatedExactBasisRatio}{0.611}
\newcommand{\LearnedScaleupJobs}{6}
\newcommand{\LearnedScaleupWidth}{256}
\newcommand{\LearnedScaleupBlocks}{6}
\newcommand{\LearnedScaleupSteps}{50000}
\newcommand{\LearnedScaleupBestStep}{40000}
\newcommand{\LearnedScaleupBestParameterError}{0.511}
\newcommand{\LearnedScaleupFinalParameterError}{0.676}

\newcommand{\LearnedScaleupFinalScoreRMS}{0.0143}

\newcommand{\LearnedScaleupMedianJobSeconds}{736}

%% file: tables/general_end_to_end_macros.tex
\newcommand{\GeneralEtoEEstimatedCases}{270}
\newcommand{\GeneralEtoEPopulationCases}{135}

\newcommand{\GeneralEtoEMaxBranches}{6}
\newcommand{\GeneralEtoEPopulationAngle}{0.0197}

\newcommand{\GeneralEtoEPopulationError}{0.132}

\newcommand{\GeneralEtoEKDEAngle}{0.0222}
\newcommand{\GeneralEtoEKDEError}{0.957}

\newcommand{\GeneralEtoELargestSigma}{0.08}
\newcommand{\GeneralEtoEKDELargestSigmaMin}{0.688}
\newcommand{\GeneralEtoEKDELargestSigmaMax}{0.810}

\newcommand{\GeneralEtoESampleSize}{131{,}072}

%% file: sections/00_abstract.tex
\begin{abstract}
At a junction, a score field can reveal weighted tangent rays, but these rays do not determine how branches bend or how density changes away from the center. This missing information matters when local geometry is used to describe continuation beyond a single point. Small-noise expansions place curvature and density in the first correction, but do not establish whether finite observations can separate the branchwise contributions when the center is estimated. We study this inverse problem through matched queries at scales $\sigma$ and $\lambda\sigma$. For a finite union of $C^{2,\alpha}$ half-branches in $\R^D$, the normalized score satisfies $F_\sigma=F_0+\sigma G+O(\sigma^{1+\alpha})$. Matched subtraction cancels the tangent term and exposes $G$, which is linear in branchwise curvature and log-density slope. Given tangent directions and weights, $G$ uniquely identifies all $sD$ branch parameters on distinct rays; a geometry-dependent condition number governs stability as rays approach one another. The $sD$ fixed component observations are also necessary. This is a scalar-information count rather than a count of vector-valued score-network evaluations. Under coarse-localization and full-rank calibration conditions, an $O(\sigma^2)$ center error adds $D$ translation modes, giving $(s+1)D$ scalar observations except for a translation-invariant full line. We derive a perturbation bound and a conditional kernel-density-estimation (KDE) rate. Experiments recover the predicted population and $N^{-1/5}$ trends and remain full rank up to $D=20$ with 16 supplied branches. A known-count frontend fitted directly to finite-noise scores composes with the inverse in $D=3$--$5$: all \GeneralEtoEPopulationCases{} population systems are full rank, with median relative jet error \GeneralEtoEPopulationError{}. Under strong population first-order error, matched responses reduce median parameter error by \ResponsePopulationStrongRatio{} times relative to naive tangent subtraction.
\end{abstract}

%% file: sections/01_introduction.tex
\section{Introduction}

Score fields are learned by score matching \citep{hyvarinen2005score}, drive noise-conditional generative models \citep{song2019generative,song2021score}, and serve as local measurements of data geometry. Near a smooth manifold, they reveal tangent spaces, intrinsic dimension, and other first-order structure \citep{stanczuk2024dimension,ventura2025manifolds,li2026rates}. A junction, however, has no single tangent space. Its first-order description is a set of outgoing directions and relative masses. For local geometric recovery, that description is incomplete: it cannot tell how each branch continues away from the junction or how mass changes along it.

At leading order, a junction is therefore represented only by weighted tangent rays. The two neighborhoods in \cref{fig:problem-intuition} have the same such representation, yet bend differently and carry different outward density trends. Rescaling toward the junction removes these differences and leaves the same leading score field $F_0$. The missing quantities survive only in the next-order term. Recovering them requires isolating a weak cross-scale signal, assigning a superposition of responses to individual branches, and controlling center error at the same asymptotic order.

\begin{figure}[t]
    \centering
    \includegraphics[width=0.88\linewidth]{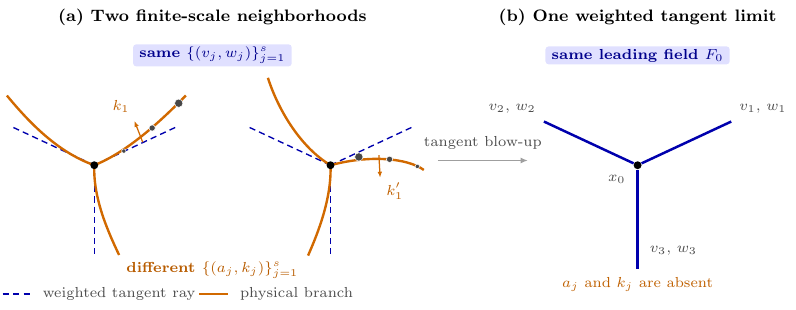}
    \caption{The same tangent geometry can hide different local structure. Tangent blow-up keeps branch directions and weights but removes curvature and outward density change.}
    \label{fig:problem-intuition}
\end{figure}

Variation across noise scales offers a route to this hidden structure. Forward expansions near boundaries and corners show how curvature and density derivatives enter the first small-noise correction \citep{brosse2026boundary}. They map known geometry to its correction, but do not establish whether a superposed correction uniquely determines every branchwise coefficient from finite observations under center error. This inverse question matters because tangent geometry cannot test whether a score representation preserves how each branch continues through a junction.

The leading tangent field can be canceled rather than estimated. Consider a finite junction of one-dimensional branches in $\R^D$. Branch $j$ has direction $v_j$, weight $w_j$, curvature vector $k_j\perp v_j$, and logarithmic density slope $a_j$; we call $(a_j,k_j)$ its \emph{branch jet}. The normalized score obeys
\begin{equation}
F_\sigma(z)=F_0(z)+\sigma G(z)+O(\sigma^{1+\alpha}),
\label{eq:intro-expansion}
\end{equation}
where $F_0$ depends only on $\{v_j,w_j\}$ and $G$ is linear in the branch jets. Curvature displaces mass normally to a tangent ray, whereas density slope changes mass along it. Gaussian smoothing preserves these distinct signatures inside $G$.

Querying the same normalized location $z$ at scales $\sigma$ and $\lambda\sigma$ exposes $G$ directly:
\begin{equation}
\cD_{\sigma,\lambda}(z)
=\frac{\lambda\sigma s_{\lambda\sigma}(x_0+\lambda\sigma z)
-\sigma s_\sigma(x_0+\sigma z)}{(\lambda-1)\sigma}
=G(z)+O(\sigma^\alpha).
\label{eq:intro-difference}
\end{equation}
The common field $F_0$ cancels. With known directions and weights, scalar component evaluations become rows of a linear inverse problem. The superposed field determines every branch jet without an angular-separation assumption, and $sD$ scalar observations are necessary and sufficient. This is an information count, not a forward-pass count: a score network usually returns all $D$ coordinates at one location. Quantitative stability is governed by the condition number, which captures ray collisions and vanishing weights.

The estimated center creates a second difficulty. Given an $O(\sigma)$ coarse localization and a full-rank weak calibration system, first-order score calibration gives $c_\sigma=x_0+\sigma^2b+O(\sigma^{2+\alpha})$. This physical error is small, but normalization promotes it to the same order as $G$. Reusing the provisional center at both scales gives
\begin{equation}
\cD^{c_\sigma}_{\sigma,\lambda}(z)
=G(z)-\lambda^{-1}\nabla F_0(z)b+O(\sigma^\alpha).
\label{eq:intro-center}
\end{equation}
Thus $b\in\R^D$ contributes $D$ known translation-shaped basis fields. Except when the tangent measure is invariant along a complete line, the branch jets and center bias remain jointly identifiable from $(s+1)D$ scalar observations, and estimating $b$ refines the center to $O(\sigma^{2+\alpha})$.

\begin{figure}[t]
    \centering
    \includegraphics[width=0.80\linewidth]{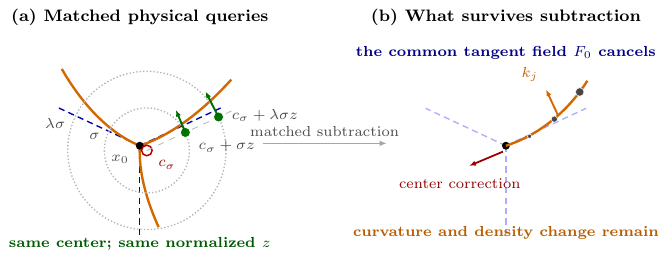}
    \caption{Matched physical queries use the same provisional center and normalized location. Their subtraction removes the shared tangent field, leaving curvature, density change, and the center correction visible.}
    \label{fig:solution-intuition}
\end{figure}

\paragraph{Contributions.}
\begin{enumerate}
    \item The superposed first correction exactly identifies every branch jet on supplied distinct rays without a fixed angular separation; quantitative stability is treated separately.
    \item A locally uniform two-term expansion and matched scale difference make the correction observable by canceling the tangent field rather than estimating and subtracting it.
    \item Fixed continuous schemes have scalar-information complexity $sD$. An augmented inverse jointly identifies the $O(\sigma^2)$ center bias and branch jets with complexity $(s+1)D$, except for the intrinsic full-line ambiguity.
    \item A perturbation bound separates finite-scale, score, center, and first-stage errors and yields a conditional KDE rate.
\end{enumerate}

\paragraph{Evidence.}
Across \PopulationInstances{} population fits in $D\in\{2,3,5\}$, the median remainder exponent is \PopulationMedianSlope{}, matching the predicted $O(\sigma)$ order. Controlled-perturbation scale-up remains full rank through $D=\ScaleLargestDimension{}$ and \ScaleLargestBranches{} supplied branches. The phase map shows a \StabilityErrorRatio{}-fold error contrast between near-collision and well-separated corners, and matched responses are \ResponsePopulationStrongRatio{} times more accurate than naive subtraction under strong population first-order error.

A known-count frontend fitted directly to finite-noise scores runs end to end in $D=3$--$5$. All \GeneralEtoEPopulationCases{} population systems are full rank, with median relative jet error \GeneralEtoEPopulationError{}. KDE scores retain a median maximum tangent error of \GeneralEtoEKDEAngle{} radians and provide measurable jet recovery with median relative error \GeneralEtoEKDEError{} under a fixed \GeneralEtoESampleSize{}-sample budget; at $\sigma=\GeneralEtoELargestSigma{}$, the dimensionwise medians improve to \GeneralEtoEKDELargestSigmaMin{}--\GeneralEtoEKDELargestSigmaMax{}. For learned scores, median jet error reaches its minimum before score RMS. Together, these experiments test finite observations, imperfect first-order geometry, sampling, and learned scores.

%% file: sections/02_model_expansion.tex
\section{Model and Second-Order Target}
\label{sec:model}

First-order geometry consists of tangent directions and relative masses; the second-order target is curvature and density change. Let $x_0\in\R^D$ be the junction center. In a neighborhood of $x_0$, assume
\begin{equation}
\mu_{\mathrm{loc}}
=\sum_{j=1}^s (\gamma_j)_\#\bigl(\rho_j(r)\dd r\bigr),
\qquad \gamma_j:[0,r_0]\to\R^D,
\label{eq:measure-model}
\end{equation}
plus a finite remainder $\mu_{\mathrm{far}}$ supported away from $x_0$. The curves are parameterized by arc length and satisfy
\begin{align}
\gamma_j(0)&=x_0,
&\gamma'_j(0)&=v_j\in\Sph^{D-1},
&\gamma''_j(0)&=k_j\in v_j^\perp,
\label{eq:curve-jet}\\
\rho_j(0)&=w_j>0,
&a_j&=\partial_r\log\rho_j(0).
\label{eq:density-jet}
\end{align}
Distinct branches have distinct $v_j$. Arc-length parameterization implies $k_j\perp v_j$.

\begin{assumption}[Uniform jet remainder]
\label{ass:jet}
For some $0<\alpha\le 1$ and finite constants $L_\gamma,L_\rho$,
\begin{align}
\norm{\gamma_j(r)-x_0-rv_j-\tfrac12r^2k_j}
&\le L_\gamma r^{2+\alpha},
\label{eq:curve-rem}\\
\abs{\rho_j(r)-w_j(1+a_jr)}
&\le L_\rho r^{1+\alpha}
\label{eq:density-rem}
\end{align}
for $0\le r\le r_0$, and
$\operatorname{dist}(x_0,\supp\mu_{\mathrm{far}})\ge\Delta>0$.
\end{assumption}

The weighted tangent measure is
\begin{equation}
\nu_0=\sum_{j=1}^s w_j\int_0^\infty \delta_{uv_j}\dd u.
\label{eq:tangent-measure}
\end{equation}
A score cannot see its total scale, so the weights may be normalized to sum to one. The known first-order geometry is $\{v_j,w_j\}_{j=1}^s$; the target is
\begin{equation}
\cJ_2(\mu,x_0)=\{(a_j,k_j):j=1,\ldots,s\}.
\label{eq:jet-target}
\end{equation}
Here $k_j$ measures bending and $a_j$ measures outward log-density change.

Let $p_\sigma=\mu*\varphi_\sigma$, where $\varphi_\sigma$ is the centered Gaussian density with covariance $\sigma^2\Id_D$, and define
\begin{equation}
s_\sigma(x)=\nabla_x\log p_\sigma(x),
\qquad
F_\sigma(z)=\sigma s_\sigma(x_0+\sigma z).
\label{eq:renormalized-score}
\end{equation}
Write the repeated one-dimensional Gaussian integrals as
\begin{equation}
J_{j,m}(z)=\int_0^\infty u^m
\exp\!\left(-\frac{\norm{z-uv_j}^2}{2}\right)\dd u.
\label{eq:Jjm}
\end{equation}
The leading scalar transform and its first correction are
\begin{align}
Q_0(z)
&=\sum_{j=1}^s w_jJ_{j,0}(z),
\label{eq:Q0}\\
Q_1(z)
&=\sum_{j=1}^s w_j\left[
    a_jJ_{j,1}(z)
    +\frac12\ip{k_j}{z}J_{j,2}(z)
\right].
\label{eq:Q1}
\end{align}

\section{Why Second-Order Geometry Appears Across Scale}
\label{sec:expansion}

The mechanism is visible before the formal expansion. At physical radius $r=\sigma u$, both the density change $a_jr$ and the normal displacement $r^2k_j/2$, after division by the observation scale $\sigma$, contribute at order $\sigma$. Gaussian smoothing therefore places both effects in the same correction $G$.

\begin{figure}[t]
    \centering
    \includegraphics[width=0.70\linewidth]{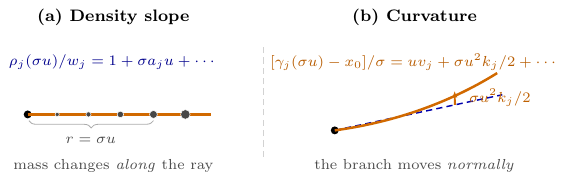}
    \caption{Two distinct order-$\sigma$ signatures. Density slope changes mass along a tangent ray, whereas curvature moves the branch normally away from it.}
    \label{fig:order-sigma-mechanism}
\end{figure}

\begin{theorem}[Second-order junction expansion]
\label{thm:expansion}
Under \cref{ass:jet}, for every compact $\cQ\subset\R^D$ and every fixed integer $m\ge 0$, there are $C_{m,\cQ}<\infty$ and $\sigma_0>0$ such that
\begin{equation}
\norm{\widetilde Q_\sigma-Q_0-\sigma Q_1}_{C^m(\cQ)}
\le C_{m,\cQ}\sigma^{1+\alpha},
\qquad 0<\sigma\le\sigma_0,
\label{eq:density-expansion}
\end{equation}
where $\widetilde Q_\sigma$ is the locally rescaled Gaussian convolution after removal of its $z$-independent normalization. Consequently, for $m\ge1$,
\begin{equation}
\norm{F_\sigma-F_0-\sigma G}_{C^{m-1}(\cQ)}
\le C'_{m,\cQ}\sigma^{1+\alpha},
\label{eq:score-expansion}
\end{equation}
with
\begin{equation}
F_0=\nabla\log Q_0,
\qquad
G=\nabla\left(\frac{Q_1}{Q_0}\right).
\label{eq:F0G}
\end{equation}
\end{theorem}

$F_0$ records the tangent rays, whereas $G$ superposes all branch jets. The expansion is forward; \cref{sec:inverse} proves that the superposition is uniquely invertible.

The proof substitutes $r=\sigma u$ into the curve and density expansions, expands the Gaussian kernel once, and integrates under a uniform Gaussian envelope. The far-away measure is exponentially small, and positivity of $Q_0$ permits the logarithm. \Cref{app:expansion} gives the complete derivative and remainder argument.

For a fixed scale ratio $\lambda>1$, define
\begin{equation}
\cD_{\sigma,\lambda}(z)
=\frac{F_{\lambda\sigma}(z)-F_\sigma(z)}{(\lambda-1)\sigma}.
\label{eq:two-scale-def}
\end{equation}
The two scores are evaluated at $x_0+\lambda\sigma z$ and $x_0+\sigma z$, respectively, so they share the same normalized location $z$.

\begin{corollary}[Renormalized scale derivative]
\label{cor:scale-difference}
For fixed $1<\lambda\le\lambda_{\max}$ and compact $\cQ$,
\begin{equation}
\sup_{z\in\cQ}\norm{\cD_{\sigma,\lambda}(z)-G(z)}
\le C_{\cQ,\lambda}\sigma^\alpha.
\label{eq:two-scale-rate}
\end{equation}
Equivalently,
$G(z)=\partial_\sigma[\sigma s_\sigma(x_0+\sigma z)]|_{\sigma=0^+}$.
\end{corollary}

Thus matched subtraction removes the common tangent field $F_0$ and exposes $G$ without differentiating a noisy score.

If an exact tangent score were available, $(F_\sigma-F_0)/\sigma$ would have the same limit. The matched difference is preferable in practice because it cancels the leading field present in the observations themselves.

%% file: sections/03_identifiability_queries.tex
\section{Recovering Branch Jets}
\label{sec:inverse}

With the two-scale response in hand, two questions remain: does it uniquely determine the branch jets, and how many scalar observations are needed? Both become finite-dimensional once the directions and weights are known.

Choose an orthonormal basis
$N_j=[n_{j,1},\ldots,n_{j,D-1}]$ for $v_j^\perp$ and write
$k_j=N_j\kappa_j$, with $\kappa_j\in\R^{D-1}$. Define
\begin{align}
A_j(z)
&=\nabla\left(\frac{w_jJ_{j,1}(z)}{Q_0(z)}\right),
\label{eq:Afield}\\
K_{j,r}(z)
&=\nabla\left(
\frac{w_j\ip{n_{j,r}}{z}J_{j,2}(z)}{2Q_0(z)}
\right),
\quad r=1,\ldots,D-1.
\label{eq:Kfield}
\end{align}
Then
\begin{equation}
G(z)=\sum_{j=1}^s a_jA_j(z)
+\sum_{j=1}^s\sum_{r=1}^{D-1}\kappa_{j,r}K_{j,r}(z).
\label{eq:linear-G}
\end{equation}
Thus each density slope and curvature coordinate multiplies a known response pattern.

\begin{theorem}[Jet identifiability]
\label{thm:jet-identifiability}
Fix distinct directions $v_1,\ldots,v_s\in\Sph^{D-1}$ and positive weights $w_1,\ldots,w_s$. The map
\begin{equation}
\{(a_j,k_j)\}_{j=1}^s\longmapsto G
\label{eq:jet-map}
\end{equation}
is injective over $a_j\in\R$ and $k_j\in v_j^\perp$.
\end{theorem}

The geometric reason is that density changes act along a ray, while curvature acts normally to it; a localized observation can therefore separate both from every other branch.

The proof represents the correction as a distribution on the tangent rays. Because Gaussian convolution is injective, a zero field implies a zero distribution. A test supported in a narrowing tube isolates one ray, first its normal curvature and then its tangential density change. Repeating this test proves uniqueness for distinct rays, although stability may deteriorate as two rays approach. \Cref{app:identifiability} gives the full argument.

\begin{theorem}[Sharp scalar-information complexity]
\label{thm:sharp-query}
Under the conditions of \cref{thm:jet-identifiability}, there exist $sD$ scalar component evaluations of $G$ whose values uniquely determine all branch jets. Conversely, any fixed continuous observation scheme that identifies every jet vector in a nonempty open subset of $\R^{sD}$ requires at least $sD$ scalar outputs.
\end{theorem}

Each branch contributes one density slope and $D-1$ curvature coordinates. One scalar observation is one coordinate at one location; because a network evaluation usually returns all $D$ coordinates, $sD$ is a scalar-information count, not a forward-pass count.

Point-coordinate evaluations span the dual of the $sD$-dimensional space generated by \cref{eq:Afield,eq:Kfield}, so $sD$ of them suffice. For necessity, invariance of domain rules out a continuous injection from an open subset of $\R^{sD}$ into fewer coordinates. In computation, we select well-conditioned rows from a larger candidate matrix; \cref{app:query-complexity} gives the formal proof.

\begin{algorithm}[t]
\caption{Two-scale branch-jet reconstruction}
\label{alg:recovery}
\begin{algorithmic}[1]
\REQUIRE First-order estimate $(c_\sigma,\{\hat v_j,\hat w_j\}_{j=1}^{\hat s})$, scales $\sigma$ and $\lambda\sigma$, candidate normalized grid $\mathcal Z$.
\STATE Construct the branch basis fields in \cref{eq:Afield,eq:Kfield}; for center-robust recovery, append \cref{eq:center-basis}.
\STATE Select scalar point-coordinate rows by rank-revealing QR, or retain an overdetermined design.
\STATE At each selected $z$, query the same center $c_\sigma$ at $c_\sigma+\sigma z$ and $c_\sigma+\lambda\sigma z$.
\STATE Form \cref{eq:centered-difference} and solve least squares.
\STATE Return $(\hat a_j,\hat k_j)$ and, when applicable, $\hat x_0=c_\sigma-\sigma^2\hat b$.
\end{algorithmic}
\end{algorithm}

%% file: sections/04_center_stability.tex
\section{Recovering Jets with an Imperfect Center}
\label{sec:center}

The previous section assumed the exact center $x_0$. A score-only first stage instead returns a nearby point. The key question is whether its error hides the order-$\sigma$ branch signal or can be estimated with it.

\begin{proposition}[Center accuracy after local calibration]
\label{prop:calibration-expansion}
Fix smooth localized score tests around a bounded reference offset. If the resulting exact tangent-model least-squares matrix has full column rank, then using the finite-noise score from \cref{thm:expansion} returns a physical center of the form
\begin{equation}
c_\sigma=x_0+\sigma^2b+O(\sigma^{2+\alpha})
\label{eq:calibrated-center}
\end{equation}
for some bounded $b\in\R^D$.
\end{proposition}

In normalized coordinates, the calibration error is order $\sigma$; multiplying by the noise scale makes the physical error order $\sigma^2$.

The proposition is local: a coarse procedure must first place the calibration window within $O(\sigma)$ of $x_0$. \Cref{app:center} gives the full weak-system and least-squares expansion. More generally, we allow any first stage satisfying
\begin{equation}
c_\sigma=x_0+\sigma^2b+r_\sigma,
\qquad
\norm{r_\sigma}\le C\sigma^{2+\alpha}.
\label{eq:provisional-center}
\end{equation}
Use this same physical center at both noise scales and define
\begin{equation}
\cD^{c_\sigma}_{\sigma,\lambda}(z)
=\frac{
\lambda\sigma s_{\lambda\sigma}(c_\sigma+\lambda\sigma z)
-\sigma s_\sigma(c_\sigma+\sigma z)
}{(\lambda-1)\sigma}.
\label{eq:centered-difference}
\end{equation}

\begin{figure}[t]
    \centering
    \includegraphics[width=0.70\linewidth]{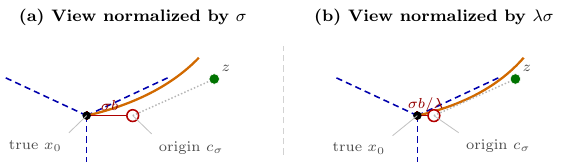}
    \caption{The same physical center error occupies different lengths in the two normalized views: $\sigma b$ at scale $\sigma$ and $\sigma b/\lambda$ at scale $\lambda\sigma$. Both panels keep the query location $z$ fixed.}
    \label{fig:center-bias-mechanism}
\end{figure}

\begin{theorem}[Center-bias expansion]
\label{thm:center-expansion}
Under \cref{ass:jet,eq:provisional-center}, uniformly on compact normalized query sets,
\begin{equation}
\cD^{c_\sigma}_{\sigma,\lambda}(z)
=G(z)-\frac{1}{\lambda}\nabla F_0(z)b+O(\sigma^\alpha),
\label{eq:center-expansion}
\end{equation}
where $\nabla F_0$ is the $D\times D$ Jacobian of the tangent score.
\end{theorem}

The shared physical offset becomes $\sigma b$ at the small scale and $\sigma b/\lambda$ at the large scale, so their mismatch has a known translation shape.

The proof expands $F_\sigma$ at the two normalized offsets and subtracts them. All quadratic shift terms are smaller than the retained order. The complete Taylor calculation is in \cref{app:center}.

Append the $D$ known translation fields
\begin{equation}
C_r(z)=-\lambda^{-1}\nabla F_0(z)e_r,
\qquad r=1,\ldots,D.
\label{eq:center-basis}
\end{equation}
The inverse now has $(s+1)D$ coordinates. The only intrinsic exception is a tangent measure that is unchanged by translation along a nonzero direction; for a finite positive ray junction, this requires a complete line formed by opposite rays with matching density.

\begin{theorem}[Center-robust identifiability and scalar information]
\label{thm:center-identifiability}
If $\nu_0$ is not invariant under any nonzero translation, the map
\begin{equation}
(\{a_j,k_j\}_{j=1}^s,b)
\longmapsto
G-\lambda^{-1}\nabla F_0b
\label{eq:augmented-map}
\end{equation}
is injective. Consequently, $(s+1)D$ fixed scalar component evaluations are sufficient, and fewer than $(s+1)D$ fixed continuous scalar observations cannot identify every parameter vector in an open set.
\end{theorem}

Unless the junction itself is translation-invariant along a full line, a center shift cannot imitate branchwise curvature and density changes.

A narrowing-tube test first removes curvature and center components normal to each ray. The endpoint term can cancel only in the full-line case. Otherwise $b=0$, and \cref{thm:jet-identifiability} removes the density terms. \Cref{app:center} gives the proof and formal invariance criterion.

Estimating $b$ immediately refines the center:
\begin{equation}
\hat x_0=c_\sigma-\sigma^2\hat b.
\label{eq:center-refinement}
\end{equation}
If $\hat b-b=O(\sigma^\alpha)$, then
$\hat x_0-x_0=O(\sigma^{2+\alpha})$.

\section{Error Propagation and Finite Samples}
\label{sec:stability}

Finite-scale approximation, score noise, and first-stage geometry error determine accuracy. Let $p=sD$ for a known center and $p=(s+1)D$ for joint center recovery. Select $M\ge p$ scalar evaluations, stack the exact basis rows into $B\in\R^{M\times p}$, and let $y$ be the corresponding two-scale response for parameter vector $\theta$.

\begin{theorem}[Deterministic perturbation bound]
\label{thm:perturbation}
Assume $\sigma_{\min}(B)=\gamma>0$. Let the implemented design be $\widehat B=B+E$ and the observed response be $\widehat y=y+e$. If $\norm{E}_{\op}<\gamma$, the least-squares estimator obeys
\begin{equation}
\norm{\widehat\theta-\theta}
\le
\frac{
 C_{\mathrm{bias}}\sigma^\alpha
 +\norm{e}
 +\norm{E}_{\op}\norm{\theta}
}{\gamma-\norm{E}_{\op}}.
\label{eq:ls-bound}
\end{equation}
If the two normalized score fields have stacked errors $\varepsilon_\sigma$ and $\varepsilon_{\lambda\sigma}$, and the center remainder in \cref{eq:provisional-center} is $r_\sigma$, one may take
\begin{equation}
\norm{e}
\le
\frac{\varepsilon_\sigma+\varepsilon_{\lambda\sigma}}{(\lambda-1)\sigma}
+C_{\mathrm{ctr}}\frac{\norm{r_\sigma}}{\sigma^2}.
\label{eq:response-error}
\end{equation}
Errors in recovered directions and weights enter through $E$ by smooth dependence of the basis on the first-order geometry over any declared separated, positive-weight class.
\end{theorem}

The bound makes conditioning explicit: finite-scale bias and response noise are amplified by the inverse smallest singular value of the selected design.

The bound follows from the pseudoinverse identity and Weyl's inequality; \cref{app:stability} gives the proof. Its main message is the decomposition: scale differencing divides normalized-score noise by $\sigma$, while an unmodeled physical center remainder is divided by $\sigma^2$.

\begin{corollary}[Kernel-density score rate]
\label{cor:kde-rate}
Assume a fixed finite query design, a positive lower bound on local tangent mass, bounded branch jets, and the conditions of \cref{thm:perturbation}. For $N$ independent samples and fixed $\lambda>1$, the score of the empirical Gaussian kernel density estimate (KDE) gives, with probability at least $1-\delta$,
\begin{equation}
\norm{\widehat\theta-\theta}
\le C\left[
\sigma^\alpha
+\sqrt{\frac{\log(M/\delta)}{N\sigma^3}}
+\frac{\log(M/\delta)}{N\sigma^2}
+\text{first-stage error}
\right].
\label{eq:kde-rate}
\end{equation}
Ignoring the lower-order Bernstein term and assuming a comparably accurate first stage,
\begin{equation}
\sigma_{\mathrm{opt}}
\asymp \left(\frac{\log(M/\delta)}{N}\right)^{1/(2\alpha+3)},
\qquad
\norm{\widehat\theta-\theta}
=O_{\mathbb P}\!\left(
\frac{\log(M/\delta)}{N}
\right)^{\alpha/(2\alpha+3)}.
\label{eq:optimal-rate}
\end{equation}
\end{corollary}

The variance term reflects two losses: only $N\sigma$ samples fall locally, and the scale difference contributes another factor $1/\sigma$.

The rate balances the deterministic remainder $\sigma^\alpha$ against the scale-differenced sampling fluctuation. A learned score enters the same response bound through its errors at both scales. Thus cross-scale consistency, rather than accuracy at either scale alone, is the quantity that controls second-order recovery.

%% file: sections/05_experiments.tex
\section{Experiments}
\label{sec:experiments}

Each experiment tests one proof interface: finite-scale expansion, query conditioning, sampling, first-stage composition, or learned-score consistency. All results use the retained \textbf{\ResultLabel} configuration; \cref{app:experiments} gives full grids, seeds, hardware, and tables.

\paragraph{Finite-scale expansion.}
Across \PopulationInstances{} population fits in $D\in\{2,3,5\}$, the median remainder exponent is \PopulationMedianSlope{} (5th--95th percentile \PopulationSlopeQfive{}--\PopulationSlopeQninetyfive{}), with median log--log $R^2$ \PopulationMedianRtwo{}. Exact- and $O(\sigma^2)$-biased-center runs follow the predicted order in \cref{thm:expansion,thm:center-expansion}, providing an independent numerical check of the analytic results (\cref{fig:paper-main}).

\paragraph{Scalar information versus stable designs.}
At response-noise RMS $10^{-3}$ and twice the parameter count, determinant-maximizing (D-optimal) row selection attains median relative error \QueryDOptMedianRelativeError{}, \QueryRandomDOptErrorRatio{} times lower than random selection. Overdetermination and conditioning-aware rows increase the smallest singular value: \cref{thm:sharp-query} gives the information limit, whereas design controls noise amplification.

\paragraph{How geometry controls stability and scale.}
The 42-cell phase map in \cref{fig:revision-main}a has 20 seeds per cell. From $(3^\circ,0.01)$ to $(45^\circ,0.2)$, median error falls from \StabilityNearCollisionError{} to \StabilityWellSeparatedError{} (\StabilityErrorRatio{}-fold) and condition number from \StabilityNearCollisionCondition{} to \StabilityWellSeparatedCondition{}. All 200 scale-up systems remain full rank through $D=\ScaleLargestDimension{}$, \ScaleLargestBranches{} branches, \ScaleLargestParameters{} unknowns, and \ScaleLargestQueries{} queries. These runs use perturbed supplied geometry; the frontend study below covers $D=3$--$5$. In the largest case, strong first-order perturbation raises error from \ScaleLargestExactError{} to \ScaleLargestStrongError{}, so first-stage accuracy controls error even when rank is preserved.

\paragraph{Why match two scales?}
Naive subtraction divides first-order error by the base scale. Its population error is \ResponsePopulationMildRatio{}, \ResponsePopulationModerateRatio{}, and \ResponsePopulationStrongRatio{} times the matched error under mild, moderate, and strong perturbations; the smallest-scale strong ratio is \ResponsePopulationSmallSigmaStrongRatio{}. KDE noise reduces it to \ResponseKDEStrongRatio{}. The exact-$F_0$ curve is an idealized reference for perfect nuisance removal; with estimated geometry, matched responses avoid the $1/\sigma$ amplification (\cref{fig:revision-main}b).

\begin{figure}[t]
\centering
\includegraphics[width=0.86\linewidth]{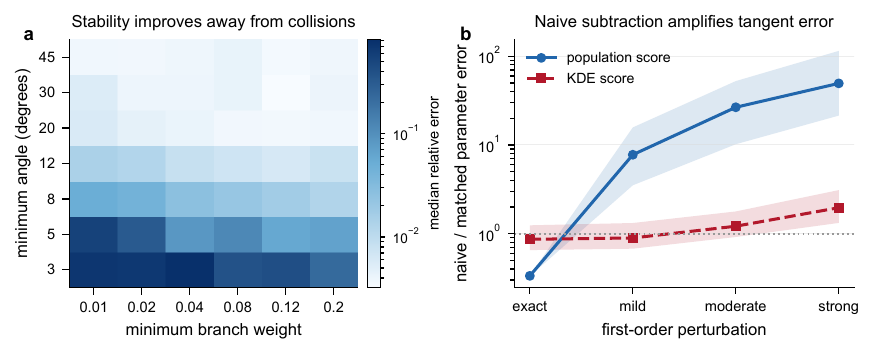}
\caption{Stability and response construction. (a) Median error over 20 seeds per angle--weight cell grows as rays collide or weights vanish. (b) Paired naive/matched error ratio; values above one favor the matched response. Bands are interquartile ranges.}
\label{fig:revision-main}
\end{figure}

\paragraph{Finite-sample kernel density estimates.}
Across four planar junctions and twenty seeds, all seed-bootstrap 95\% intervals contain the predicted $N^{-1/5}$ slope, and error falls by \KDEEndpointReductionMinPercent{}--\KDEEndpointReductionMaxPercent{}\%. Seed halves select bandwidth multipliers within $0.4$ in \KDEWithinPointFourAgreementPercent{}\% of comparisons; the dense $[1.8,3.2]$ sweep identifies a broad stable operating region across geometries (\cref{fig:kde-grid}).

\paragraph{Estimated first-order frontends.}
On planar junctions, a Fourier-based frontend estimates branch count, directions, and weights. Blind count accuracy is \EndToEndPopulationBlindCountPercent{}\% for population and \EndToEndKDEBlindCountPercent{}\% for KDE scores. A separate known-count frontend fits the same first-order quantities directly from finite-noise scores in $D=3$--$5$ with up to \GeneralEtoEMaxBranches{} branches. All \GeneralEtoEEstimatedCases{} second-stage systems are full rank. Population runs have median maximum tangent error \GeneralEtoEPopulationAngle{} radians and median relative jet error \GeneralEtoEPopulationError{}. With KDE scores, these medians are \GeneralEtoEKDEAngle{} and \GeneralEtoEKDEError{}; at $\sigma=\GeneralEtoELargestSigma{}$, the dimensionwise median jet errors range from \GeneralEtoEKDELargestSigmaMin{} to \GeneralEtoEKDELargestSigmaMax{} (\cref{tab:general-end-to-end}). Changing the query rows has negligible effect (median error ratio \DiagnosticQueryRowRatio{}). In population runs, the estimated basis absorbs part of the finite-scale remainder, giving an estimated/exact-basis error ratio of \DiagnosticPopulationEstimatedExactBasisRatio{} (\cref{fig:revision-diagnostics}).

\paragraph{Learned scores.}
Across \LearnedJobCount{} base jobs, the training trajectories separate single-scale score fit from recovered-jet accuracy. In the enlarged hard-case study (width \LearnedScaleupWidth{}, \LearnedScaleupBlocks{} blocks, \LearnedScaleupSteps{} updates), median jet error reaches \LearnedScaleupBestParameterError{} at \LearnedScaleupBestStep{}; score RMS continues to improve to \LearnedScaleupFinalScoreRMS{} at the final checkpoint, where jet error is \LearnedScaleupFinalParameterError{}. This controlled synthetic experiment isolates cross-scale consistency as the mechanism used by the inverse (\cref{fig:learned-grid,fig:revision-diagnostics}).

The support material contains the experiment code and its full configuration; each entry point writes its numerical observations as CSV files (\cref{app:experiments}).

%% file: sections/06_related_limitations.tex
\section{Related Work}
\label{sec:related}

Classical work reconstructs manifolds and estimates tangent spaces or curvature from samples \citep{federer1959curvature,preiss1987geometry,niyogi2008homology,ozertem2011principal,genovese2012minimax,aamari2019nonasymptotic}; score methods recover first-order and smooth-support geometry, including weighted singular tangents \citep{stanczuk2024dimension,ventura2025manifolds,diepeveen2025pullback,kharitenko2026landing,li2026rates,weightedtangent2026}. Our target is instead the branchwise second-order continuation of a singular junction on supplied rays.

Denoising identities expose local scores \citep{vincent2011connection,alain2014regularized}, and forward asymptotics derive their density and curvature corrections \citep{brosse2026boundary,zhang2026manifold,rawal2026rao}. Those results map known geometry to a correction. We solve the converse coefficient problem: uniquely decomposing the superposed correction with finite scalar information while accounting for center translation. Unlike Prony and super-resolution, the unknowns are coefficients on supplied rays, not atom locations \citep{kunis2016prony,kunis2019sphere,candes2014superresolution}.

\section{Limitations and Conclusion}
\label{sec:limitations}

\paragraph{Limitations.}
The theory covers zero-thickness $C^{2,\alpha}$ half-branches, positive $C^{1,\alpha}$ densities, Gaussian smoothing, two known scales, branch correspondence, and supplied first-order geometry. The condition number quantifies the effect of ray collisions and vanishing weights, and the KDE rate is conditioned on first-stage accuracy. Experiments separate full-pipeline recovery with known counts in $D=3$--$5$ and blind planar counts from second-stage scaling through $D=20$ and 16 branches with supplied perturbed geometry. The full-line case exhibits the intrinsic translation symmetry, while learned-score runs isolate cross-scale consistency on controlled synthetic junctions.

\paragraph{Conclusion.}
The first score correction has an inverse role: on supplied distinct rays, its superposed field uniquely determines how every branch bends and how its density changes. Matched subtraction makes the correction observable with sharp scalar-information complexity, while translation modes accommodate imperfect localization. Known-count experiments carry score-estimated first-order geometry through the inverse in $D=3$--$5$, and the supplied-geometry interface scales to larger systems. The stability map, matched-response comparison, and learned-score trajectories identify the geometric and statistical conditions that govern accuracy. Branch jets thus test whether a score field retains junction continuation and provide a target for thicker supports, higher-dimensional strata, pretrained models, and adaptive queries.

%% file: appendices/A_expansion.tex
\section{Proof of the two-term expansion}
\label{app:expansion}

The local argument below tracks every power of $\sigma$. Constants may depend on the compact query set, the branch count, the jet bounds, $r_0$, $\Delta$, and the far mass, but not on $\sigma$.

\subsection{A Gaussian envelope}

\begin{lemma}[Polynomial Gaussian envelope]
\label{lem:envelope}
Fix a compact $\cQ\subset\R^D$. There are $c>0$ and constants $C_{m,\cQ}$ such that the following holds. Let $v\in\Sph^{D-1}$, $u\ge0$, and let $q_t=uv+t\xi$ for $0\le t\le1$, with $\norm{\xi}\le u/2$. For every fixed mixed-derivative order $m$ of $K_z(q)=e^{-\norm{z-q}^2/2}$ in $(z,q)$,
\begin{equation}
\abs{\partial^m K_z(q_t)}
\le C_{m,\cQ}(1+u)^{m}e^{-cu^2},
\qquad z\in\cQ.
\label{eq:envelope}
\end{equation}
The same form holds for vector and matrix derivatives.
\end{lemma}

\begin{proof}
Let $R=\sup_{z\in\cQ}\norm{z}$. Since $u/2\le\norm{q_t}\le3u/2$, for $u\ge4R$ we have $\norm{z-q_t}\ge u/4$. For bounded $u$, every polynomial factor is absorbed into the constant. Derivatives of a Gaussian equal a polynomial in $z-q_t$ times the same Gaussian, giving the claim.
\end{proof}

\subsection{Expansion of the rescaled convolution}

After changing variables $r=\sigma u$, remove the common factor $(2\pi)^{-D/2}\sigma^{1-D}$ and write
\begin{equation}
\widetilde Q_\sigma(z)
=\sum_{j=1}^s\int_0^{r_0/\sigma}
\rho_j(\sigma u)K_z(q_{j,\sigma}(u))\dd u
+R_{\mathrm{far},\sigma}(z),
\label{eq:Qsigma-rescaled}
\end{equation}
where $q_{j,\sigma}(u)=(\gamma_j(\sigma u)-x_0)/\sigma$.

By \cref{eq:curve-rem},
\begin{equation}
q_{j,\sigma}(u)=uv_j+\delta_{j,\sigma}(u),
\qquad
\delta_{j,\sigma}(u)
=\frac{\sigma}{2}u^2k_j+\epsilon_{j,\sigma}(u),
\quad
\norm{\epsilon_{j,\sigma}(u)}
\le L_\gamma\sigma^{1+\alpha}u^{2+\alpha}.
\label{eq:delta}
\end{equation}
Decrease $r_0$ if needed so that $\norm{\delta_{j,\sigma}(u)}\le u/2$ for $0\le u\le r_0/\sigma$. Taylor's theorem and \cref{lem:envelope} give, uniformly in $z\in\cQ$,
\begin{align}
K_z(q_{j,\sigma}(u))
&=K_z(uv_j)
+\nabla_qK_z(uv_j)^\trans\delta_{j,\sigma}(u)
+R_{j,\sigma}(z,u),
\label{eq:kernel-taylor}\\
\abs{R_{j,\sigma}(z,u)}
+\norm{\nabla_zR_{j,\sigma}(z,u)}
&\le C\left[
\sigma^{1+\alpha}P_1(u)+\sigma^2P_2(u)
\right]e^{-cu^2}
\le C\sigma^{1+\alpha}P(u)e^{-cu^2}.
\label{eq:kernel-remainder}
\end{align}
The last inequality uses $0<\alpha\le1$ and $\sigma\le1$. Since
$\nabla_qK_z(uv_j)=(z-uv_j)K_z(uv_j)$ and $k_j\perp v_j$,
\begin{equation}
\nabla_qK_z(uv_j)^\trans\frac{\sigma}{2}u^2k_j
=\frac{\sigma}{2}u^2\ip{k_j}{z}K_z(uv_j).
\label{eq:curvature-term}
\end{equation}

Likewise, \cref{eq:density-rem} gives
\begin{equation}
\rho_j(\sigma u)
=w_j(1+\sigma a_ju)+r_{j,\sigma}(u),
\qquad
\abs{r_{j,\sigma}(u)}
\le L_\rho\sigma^{1+\alpha}u^{1+\alpha}.
\label{eq:density-taylor}
\end{equation}
Multiplying \cref{eq:kernel-taylor,eq:density-taylor}, collecting the order-zero and order-$\sigma$ terms, and applying the envelope to cross terms yields
\begin{equation}
\rho_j(\sigma u)K_z(q_{j,\sigma}(u))
=w_jK_z(uv_j)
+\sigma w_j\left[
 a_ju+\frac12u^2\ip{k_j}{z}
\right]K_z(uv_j)
+\mathcal R_{j,\sigma}(z,u),
\label{eq:integrand-expansion}
\end{equation}
with
\begin{equation}
\abs{\mathcal R_{j,\sigma}(z,u)}
+\norm{\nabla_z\mathcal R_{j,\sigma}(z,u)}
\le C\sigma^{1+\alpha}P(u)e^{-cu^2}.
\label{eq:integrand-rem}
\end{equation}
The right-hand side is integrable on $[0,\infty)$. Extending the tangent integrals from $r_0/\sigma$ to infinity contributes a Gaussian tail smaller than every power of $\sigma$.

For the far measure, $\norm{x-x_0}\ge\Delta$. On a compact normalized query set and for small $\sigma$,
$\norm{(x-x_0)/\sigma-z}\ge\Delta/(2\sigma)$. Hence the rescaled far denominator and its $z$-gradient are bounded by a polynomial in $\sigma^{-1}$ times $e^{-\Delta^2/(8\sigma^2)}$, again smaller than $\sigma^{1+\alpha}$. Integrating \cref{eq:integrand-expansion} proves \cref{eq:density-expansion}.

Finally, $Q_0$ is strictly positive and has a positive minimum on $\cQ$. Write
$\widetilde Q_\sigma=Q_0[1+\sigma Q_1/Q_0+R_\sigma/Q_0]$. For every fixed $m$, the preceding envelope applies to all $z$-derivatives through order $m$. The map $x\mapsto\log x$ is smooth on a compact interval bounded away from zero, so composition in $C^m$ gives
\begin{equation}
\log\widetilde Q_\sigma
=\log Q_0+\sigma Q_1/Q_0+O_{C^m}(\sigma^{1+\alpha})
\label{eq:log-expansion}
\end{equation}
which proves \cref{eq:score-expansion}.

%% file: appendices/B_identifiability.tex
\section{Distributional proof of jet identifiability}
\label{app:identifiability}

Let $K(z)=e^{-\norm z^2/2}$. The tangent transform is $Q_0=K*\nu_0$. Associate to the jet parameters the tempered distribution
\begin{equation}
\ip{\eta}{\phi}
=\sum_{j=1}^s w_j\int_0^\infty
\left[
 a_ju\phi(uv_j)
 +\frac12u^2\ip{k_j}{\nabla\phi(uv_j)}
\right]\dd u.
\label{eq:eta-distribution}
\end{equation}
Direct evaluation on the translated kernel $y\mapsto K(z-y)$ gives
\begin{align}
(K*\eta)(z)
&=\sum_jw_j\int_0^\infty\left[
 a_juK(z-uv_j)
 +\frac12u^2\ip{k_j}{\nabla_yK(z-y)|_{y=uv_j}}
\right]\dd u
\nonumber\\
&=\sum_jw_j\int_0^\infty\left[
 a_ju+\frac12u^2\ip{k_j}{z-uv_j}
\right]K(z-uv_j)\dd u
=Q_1(z),
\label{eq:convolution-eta}
\end{align}
where $\ip{k_j}{v_j}=0$.

Suppose $G=0$. Since $Q_0>0$ and $\R^D$ is connected,
$\nabla(Q_1/Q_0)=0$ implies $Q_1=CQ_0$. Thus
$K*(\eta-C\nu_0)=0$. Taking Fourier transforms in the space of tempered distributions gives
\begin{equation}
\widehat K(\xi)\bigl(\widehat\eta-C\widehat\nu_0\bigr)=0.
\label{eq:fourier-injectivity}
\end{equation}
The Gaussian multiplier $\widehat K(\xi)$ is everywhere positive, so $\eta=C\nu_0$.

Fix branch $j$ and an interval $I=(u_-,u_+)\Subset(0,\infty)$ such that the compact ray segment $\{uv_j:u\in\overline I\}$ has a tubular neighborhood disjoint from the other rays and the origin. In local coordinates $y=uv_j+n$, $n\in v_j^\perp$, choose
$\phi(y)=\chi(u)\psi(n)$ with $\psi(0)=0$ and arbitrary $\nabla\psi(0)$. Then $\nu_0(\phi)=0$, and all branches except $j$ vanish by support. Equality $\eta=C\nu_0$ gives
\begin{equation}
\frac{w_j}{2}\int_Iu^2\chi(u)\ip{k_j}{\nabla\psi(0)}\dd u=0.
\label{eq:k-test}
\end{equation}
Arbitrary $\chi$ and $\nabla\psi(0)$ imply $k_j=0$.

Now take $\psi(0)=1$ and $\nabla\psi(0)=0$. Then
\begin{equation}
w_j\int_I(a_ju-C)\chi(u)\dd u=0
\label{eq:a-test}
\end{equation}
for every $\chi\in C_c^\infty(I)$. Hence $a_ju=C$ on $I$. Since $I$ contains more than one point, $a_j=0$ and $C=0$. Repeating for every branch proves the nullspace is trivial. Applying the result to the difference of two parameter collections proves \cref{thm:jet-identifiability}.

%% file: appendices/C_query_complexity.tex
\section{Finite evaluations and the query lower bound}
\label{app:query-complexity}

\begin{lemma}[Evaluation basis]
\label{lem:evaluation-basis}
Let $f_1,\ldots,f_p:X\to\R^D$ be linearly independent as functions. Then there exist pairs $(x_m,r_m)\in X\times\{1,\ldots,D\}$, $m=1,\ldots,p$, such that the matrix
$[e_{r_m}^\trans f_\ell(x_m)]_{m,\ell=1}^p$ is nonsingular.
\end{lemma}

\begin{proof}
Let $V=\Span\{f_1,\ldots,f_p\}$. Point-coordinate evaluations $L_{x,r}(f)=e_r^\trans f(x)$ separate points of $V$: if every such functional vanishes on $f\in V$, then $f$ is the zero function. Therefore their restrictions span $V^*$. Select $p$ of them forming a basis of $V^*$; their matrix in the basis $f_1,\ldots,f_p$ is nonsingular.
\end{proof}

By \cref{thm:jet-identifiability}, the $p=sD$ fields in \cref{eq:Afield,eq:Kfield} are linearly independent, so \cref{lem:evaluation-basis} proves the upper bound in \cref{thm:sharp-query}. For the lower bound, restrict the parameter vector to a nonempty open box in $\R^{sD}$. Any fixed continuous scheme with $M$ scalar outputs induces a continuous injective map from that box to $\R^M$. If $M<sD$, compose with the standard coordinate embedding $\R^M\hookrightarrow\R^{sD}$. Invariance of domain \citep{munkres2000topology} would make the image open in $\R^{sD}$, although it lies in a lower-dimensional coordinate subspace, a contradiction.

%% file: appendices/D_center.tex
\section{Center Calibration and Center-Robust Identifiability}
\label{app:center}

\subsection{Why weak calibration has a physical $O(\sigma^2)$ error}
Let $c_\sigma^{\rm ref}=x_0-\sigma b_0$ with bounded $b_0$, and use normalized coordinates $y=(x-c_\sigma^{\rm ref})/\sigma$. The normalized finite-noise score can be written
\begin{equation}
u_\sigma(y)=F_0(y-b_0)+\sigma G(y-b_0)+O_{C^1}(\sigma^{1+\alpha})
\label{eq:calibration-score-expansion}
\end{equation}
on every compact set containing the test supports. For a smooth localized test $\psi$, the corresponding weak-calibration row and response are
\begin{align}
a_\psi(u)&=\left(\int \psi(y)u(y)\dd y,\;\int\psi(y)\dd y\right),
\label{eq:weak-row}\\
r_\psi(u)&=\int\left[-\nabla\psi(y)^\top u(y)+\psi(y)\{\norm{u(y)}^2+y^\top u(y)+D\}\right]\dd y.
\label{eq:weak-response}
\end{align}
Both maps are smooth polynomial functionals of $u$ on the compact test supports. Substitution of \cref{eq:calibration-score-expansion} therefore gives
\begin{equation}
A_\sigma=A_0+\sigma A_1+O(\sigma^{1+\alpha}),
\qquad
r_\sigma=r_0+\sigma r_1+O(\sigma^{1+\alpha}).
\label{eq:weak-system-expansion-app}
\end{equation}
The exact tangent identity gives $r_0=A_0\beta_0$ with $\beta_0=(b_0,1)$ because each branch is one-dimensional. If $A_0$ has full column rank, the least-squares solution is
\begin{equation}
\widehat\beta_\sigma=(A_\sigma^\top A_\sigma)^{-1}A_\sigma^\top r_\sigma.
\end{equation}
The full-rank least-squares map is smooth in a neighborhood of $(A_0,r_0)$. Differentiating its normal equations, or expanding the pseudoinverse, yields
\begin{equation}
\widehat\beta_\sigma
=\beta_0+\sigma(A_0^\top A_0)^{-1}A_0^\top(r_1-A_1\beta_0)
+O(\sigma^{1+\alpha}),
\label{eq:weak-ls-expansion-app}
\end{equation}
where the apparent derivatives of $A_\sigma^\top A_\sigma$ cancel because $r_0=A_0\beta_0$. If $b_1$ denotes the first $D$ coordinates of the order-$\sigma$ coefficient, the physical center returned by the calibration is
\begin{equation}
c_\sigma^{\rm ref}+\sigma\widehat b_\sigma
=x_0+\sigma^2b_1+O(\sigma^{2+\alpha}).
\end{equation}
This proves \cref{prop:calibration-expansion}. It also makes the role of the local-window assumption explicit: a separate coarse procedure must place the reference within $O(\sigma)$ of $x_0$.

\subsection{Two-scale center expansion}
To prove \cref{thm:center-expansion}, put
$h_\sigma=(c_\sigma-x_0)/\sigma=\sigma b+O(\sigma^{1+\alpha})$ and
$h_{\lambda\sigma}=(c_\sigma-x_0)/(\lambda\sigma)=(\sigma/\lambda)b+O(\sigma^{1+\alpha})$. Uniform $C^1$ regularity of \cref{eq:score-expansion} and a Taylor expansion give
\begin{align}
\sigma s_\sigma(c_\sigma+\sigma z)
&=F_\sigma(z+h_\sigma)
=F_0(z)+\sigma\nabla F_0(z)b+\sigma G(z)+O(\sigma^{1+\alpha}),
\label{eq:center-small}\\
\lambda\sigma s_{\lambda\sigma}(c_\sigma+\lambda\sigma z)
&=F_{\lambda\sigma}(z+h_{\lambda\sigma})
=F_0(z)+\frac{\sigma}{\lambda}\nabla F_0(z)b
+\lambda\sigma G(z)+O(\sigma^{1+\alpha}).
\label{eq:center-large}
\end{align}
Terms quadratic in $h$ and products of $h$ with the order-$\sigma$ correction are $O(\sigma^2)$, which is contained in $O(\sigma^{1+\alpha})$ for $\alpha\le1$. Subtraction proves \cref{eq:center-expansion}.

\begin{definition}[Translation lineality]
The tangent measure $\nu_0$ has translation lineality if there exists $b\ne0$ such that $(T_{tb})_\#\nu_0=\nu_0$ for every $t\in\R$. For a finite positive ray junction, this can occur only along a complete line represented by opposite rays with matching density.
\end{definition}

For injectivity, suppose the augmented field in \cref{eq:augmented-map} is zero. Since
\begin{equation*}
\nabla F_0b=\nabla(b^\trans\nabla\log Q_0),
\end{equation*}
\begin{equation}
\nabla\left(
\frac{Q_1-\lambda^{-1}b^\trans\nabla Q_0}{Q_0}
\right)=0.
\label{eq:center-scalar}
\end{equation}
Hence
$Q_1-\lambda^{-1}b^\trans\nabla Q_0=CQ_0$. Distributional differentiation commutes with convolution, so Gaussian injectivity yields
\begin{equation}
\eta-\lambda^{-1}b\cdot\nabla\nu_0=C\nu_0.
\label{eq:augmented-distribution}
\end{equation}

Use the same tube around an interior segment of ray $j$. For a test function whose value is zero on the ray and whose normal gradient $g(u)\in v_j^\perp$ is arbitrary, \cref{eq:augmented-distribution} becomes
\begin{equation}
w_j\int_I
\ip{\tfrac12u^2k_j+\lambda^{-1}P_{v_j^\perp}b}{g(u)}\dd u=0.
\label{eq:center-normal-test}
\end{equation}
Therefore
$\tfrac12u^2k_j+\lambda^{-1}P_{v_j^\perp}b=0$ for all $u\in I$, which implies
\begin{equation}
k_j=0,
\qquad
P_{v_j^\perp}b=0
\label{eq:center-normal-zero}
\end{equation}
for every $j$. Thus $b$ is parallel to every ray direction.

With the normal part removed, $b\cdot\nabla\nu_0$ is supported at the common endpoint. Indeed, for a smooth compactly supported $\phi$,
\begin{equation}
\ip{b\cdot\nabla\nu_0}{\phi}
=-\sum_jw_j\int_0^\infty(b^\trans v_j)\frac{\dd}{\dd u}\phi(uv_j)\dd u
=\left(\sum_jw_jb^\trans v_j\right)\phi(0).
\label{eq:endpoint-atom}
\end{equation}
Neither $C\nu_0$ nor the remaining density-slope distribution has an atom at zero, so the coefficient must vanish. Under \cref{eq:center-normal-zero}, this cancellation is equivalent to $b$ lying in a translation-lineality direction of the finite ray measure. The no-lineality assumption gives $b=0$. The argument in \cref{app:identifiability} then gives $a_j=C=0$. Linear independence and the query count follow exactly as in \cref{lem:evaluation-basis} and the invariance-of-domain lower bound.

%% file: appendices/E_stability_kde.tex
\section{Perturbation Composition and Finite-Sample Score Error}
\label{app:stability}

\subsection{Deterministic least-squares composition}
Let $\theta\in\R^p$ be the branch-jet vector, with the center coefficient appended in the augmented model. After selecting $M$ scalar point-coordinate evaluations, write
\begin{equation}
y=B\theta+r_{\rm fs},
\qquad
\norm{r_{\rm fs}}\le C_{\rm bias}\sigma^\alpha,
\label{eq:finite-scale-linear-model}
\end{equation}
where $B$ is the limiting design and $r_{\rm fs}$ is the finite-scale remainder from \cref{cor:scale-difference,thm:center-expansion}. The implemented design and response are $\widehat B=B+E$ and $\widehat y=y+e$. Whenever $\widehat B$ has full column rank,
\begin{align}
\widehat\theta-\theta
&=\widehat B^\dagger(\widehat y-\widehat B\theta)
=\widehat B^\dagger(r_{\rm fs}+e-E\theta),
\label{eq:pseudoinverse-composition}\\
\norm{\widehat\theta-\theta}
&\le \norm{\widehat B^\dagger}_{\op}
\left(C_{\rm bias}\sigma^\alpha+\norm e+\norm E_{\op}\norm\theta\right).
\end{align}
Weyl's inequality gives
$\sigma_{\min}(\widehat B)\ge\sigma_{\min}(B)-\norm E_{\op}$, which proves \cref{eq:ls-bound}.

Suppose the normalized score vectors at the selected queries have stacked errors $\varepsilon_\sigma$ and $\varepsilon_{\lambda\sigma}$. Scale differencing contributes at most
\begin{equation}
\frac{\varepsilon_\sigma+\varepsilon_{\lambda\sigma}}
{(\lambda-1)\sigma}.
\label{eq:scale-difference-noise}
\end{equation}
If the provisional center has an additional physical remainder $r_\sigma$, its normalized displacements are $r_\sigma/\sigma$ and $r_\sigma/(\lambda\sigma)$. A mean-value bound for the two normalized score fields, followed by division by $(\lambda-1)\sigma$, gives
\begin{equation}
C_{\rm ctr}\frac{\norm{r_\sigma}}{\sigma^2}.
\label{eq:center-remainder-bound}
\end{equation}
This proves \cref{eq:response-error}.

Finally, for a fixed finite candidate set the map from distinct directions and positive normalized weights to every design entry is smooth: the ray-Gaussian integrals and their derivatives are dominated by polynomial Gaussian envelopes, and $Q_0$ is uniformly positive. Hence on a compact first-order class with angular separation $\delta_0>0$ and minimum weight $w_0>0$,
\begin{equation}
\norm E_{\op}\le L_B\left(\max_j\norm{\widehat v_j-v_j}
+\norm{\widehat w-w}\right)
\label{eq:basis-lipschitz}
\end{equation}
after the branch correspondence is fixed. This is the explicit ``first-stage design error'' term used in the end-to-end experiment.

\subsection{A simultaneous empirical KDE score bound}
At a physical query $x$, define
\begin{equation}
W_i=e^{-\norm{x-X_i}^2/(2\sigma^2)},
\qquad
V_i=\frac{X_i-x}{\sigma}W_i,
\label{eq:WV}
\end{equation}
with means $a=\E W_i$ and $b=\E V_i$. The normalized population and empirical KDE scores are $F=b/a$ and $\widehat F=\overline V/\overline W$. The elementary envelopes
\begin{equation}
0\le W_i\le1,
\qquad
\norm{V_i}\le e^{-1/2},
\qquad
\E\norm{V_i}^2\le\frac{2}{e}a
\label{eq:WV-envelopes}
\end{equation}
follow by maximizing $re^{-r^2/2}$ and using $r^2e^{-r^2}\le(2/e)e^{-r^2/2}$.

Let $t=\log(4M/\delta)$. Scalar Bernstein for $W_i$ and a Hilbert-space Bernstein inequality for $V_i$ \citep{pinelis1994martingales}, followed by a union bound over $M$ queries, imply with probability at least $1-\delta$ that simultaneously
\begin{align}
\abs{\overline W-a}
&\le D(a):=\sqrt{\frac{2at}{N}}+\frac{t}{3N},
\label{eq:bernstein-W}\\
\norm{\overline V-b}
&\le Q(a):=\sqrt{\frac{4at}{eN}}+\frac{4e^{-1/2}t}{3N}.
\label{eq:bernstein-V}
\end{align}
Whenever $D(a)<a$, direct ratio subtraction gives
\begin{equation}
\norm{\widehat F-F}
\le\frac{Q(a)+\norm F D(a)}{a-D(a)}.
\label{eq:ratio-kde-bound}
\end{equation}
On a declared compact normalized query set for a one-dimensional branch junction, positivity and the local density lower bound give $a\ge a_-\sigma$ for all sufficiently small $\sigma$, while $\norm F$ is bounded. If
$\sqrt{2t/(Na_-\sigma)}+t/(3Na_-\sigma)\le1/2$, \cref{eq:ratio-kde-bound} simplifies to
\begin{equation}
\max_{m\le M}\norm{\widehat F_\sigma(z_m)-F_\sigma(z_m)}
\le C\left[
\sqrt{\frac{t}{N\sigma}}+\frac{t}{N\sigma}
\right].
\label{eq:normalized-kde}
\end{equation}
The same bound holds at $\lambda\sigma$ with constants depending on fixed $\lambda$.

Substituting \cref{eq:normalized-kde} into \cref{eq:scale-difference-noise}, absorbing the fixed finite design and $\lambda$ into the constant, and adding \cref{eq:finite-scale-linear-model,eq:basis-lipschitz} yields
\begin{equation}
\norm{\widehat\theta-\theta}
\le C\left[
\sigma^\alpha
+\sqrt{\frac{\log(M/\delta)}{N\sigma^3}}
+\frac{\log(M/\delta)}{N\sigma^2}
+\varepsilon_{\rm first}
\right].
\label{eq:kde-rate-app}
\end{equation}
Here $\varepsilon_{\rm first}$ collects the center remainder and the tangent direction/weight error after multiplication by their deterministic Lipschitz constants. Under sample splitting, or under an independently validated first stage whose contribution is no larger than the displayed stochastic terms, balancing $\sigma^\alpha$ and $(N\sigma^3)^{-1/2}$ gives \cref{eq:optimal-rate}. The term $N^{-1}\sigma^{-2}$ is lower order at that bandwidth for every $\alpha>0$.

\subsection{Learned-score error units}
If a learned physical score has pointwise error
$\norm{s_\theta(x,\tau)-s_\tau(x)}\le\epsilon_{\rm net}(\tau)$, then its normalized score error is $\tau\epsilon_{\rm net}(\tau)$. Consequently the two-scale response error is bounded by
\begin{equation}
\frac{\epsilon_{\rm net}(\sigma)+\lambda\epsilon_{\rm net}(\lambda\sigma)}{\lambda-1}.
\label{eq:learned-error-units}
\end{equation}
The bound shows why a small average score loss does not certify the cross-scale derivative: the errors at the two scales must also be aligned, not simply small in isolation.

%% file: appendices/F_experimental_details.tex
\section{Experimental Design and Reproduction Details}
\label{app:experiments}

\subsection{Experiment-to-claim map}

\begin{table}[H]
\centering
\small
\begin{tabular}{p{0.25\linewidth}p{0.31\linewidth}p{0.33\linewidth}}
\toprule
Question & Script and primary output & Interpretation \\
\midrule
Does the finite-scale remainder have order $\sigma^\alpha$? & \nolinkurl{run_population.py}; parameter error and refined-center error over decreasing $\sigma$ & Numerical check of \cref{thm:expansion,thm:center-expansion} alongside the analytic proof. \\
Can the sharp query theorem be implemented stably? & \nolinkurl{run_query_design.py}; $\sigma_{\min}$, condition number, error under response noise & Compares D-optimal, pivoted-QR, random, and uniform scalar rows. \\
Does sampling create the predicted bias--variance tradeoff? & \nolinkurl{run_kde.py}; score error, two-scale response error, jet error, effective sample size & Tests the conditional trend in \cref{cor:kde-rate} on the retained twenty-seed grid. \\
Does the first-order frontend compose with the jet inverse? & \nolinkurl{run_end_to_end.py}; branch count, tangent error, jet/center proxy error & Separates known-count, blind-count, and exact-first-order interfaces. \\
Does score-estimated geometry compose beyond the plane? & \nolinkurl{run_general_end_to_end.py}; center, tangent, weight, and jet errors in $D=3$--$5$ & Tests the known-count score-to-jet pipeline without supplied directions or weights. \\
Does ordinary score training preserve the two-scale derivative? & \nolinkurl{run_learned_score.py}; score RMS, scale-difference RMS, jet error & Isolates cross-scale consistency on controlled synthetic junctions. \\
How do angle and minimum weight control stability? & \nolinkurl{run_stability_phase.py}; singular values, condition number, relative error & Separates exact identifiability from uniform stability over 42 cells. \\
Does the inverse remain tractable in larger systems? & \nolinkurl{run_scalability.py}; rank, error, timing through $D=20$ and 16 branches & Tests computation and sensitivity to first-order perturbations. \\
Why match two scales instead of subtracting a tangent estimate? & \nolinkurl{run_response_baselines.py}; paired response and parameter errors & Compares the proposed response with naive subtraction and an exact-$F_0$ diagnostic. \\
Which E4 interface produces the observed error? & \nolinkurl{run_end_to_end_diagnostics.py}; crossed basis and query-row policies & Separates first-order basis error, row choice, finite-scale remainder, and KDE noise. \\
Does a larger learned score solve the inverse mismatch? & \nolinkurl{run_learned_score.py}; six hard-case jobs through 50k updates & Width-256, six-block follow-up; checkpoint conclusions remain post hoc. \\
\bottomrule
\end{tabular}
\caption{Every experiment targets a distinct proof interface.}
\label{tab:experiment-map}
\end{table}

\subsection{Geometries and exact population evaluation}
Each synthetic branch is a constant-curvature arc
\begin{equation}
\gamma_j(r)-x_0=
\frac{\sin(\kappa_jr)}{\kappa_j}v_j+
\frac{1-\cos(\kappa_jr)}{\kappa_j}n_j,
\label{eq:constant-curvature-arc}
\end{equation}
with the continuous straight-ray limit at $\kappa_j=0$. This curve is exactly arc-length parameterized, satisfies $\gamma_j''(0)=\kappa_jn_j$, and has nonzero higher-order terms. The density is
\begin{equation}
\rho_j(r)=w_j\exp(a_jr+d_jr^2),
\label{eq:experiment-density}
\end{equation}
where the quadratic coefficient keeps the finite-radius sampling distribution well behaved without changing the target slope $a_j$. The normalized population score is evaluated as a Gaussian-weighted first-moment/zeroth-moment ratio using fixed Gauss--Legendre quadrature. It is never obtained by numerical differentiation.

The analytic design uses the posterior moment identities
\begin{equation}
F_0(z)=\E_{\Pi_z}[U]-z,
\qquad
\nabla F_0(z)=\Cov_{\Pi_z}(U)-I_D,
\label{eq:tangent-posterior-identities}
\end{equation}
and corresponding differentiated ray moments for the density and curvature columns. The implemented design uses this column order and sign convention, the center coefficient $-1/\lambda$, exact planar tangent reconstruction, and generic-rank constructions in $D\in\{2,3,5\}$.

\subsection{Population and finite-query matrices}
The population matrix crosses ambient dimension, branch count, random seed, scale ratio, true versus $O(\sigma^2)$-biased center, and decreasing base scale. Directions are generated with a minimum angular separation, weights are positive and normalized, and curvature/density-slope coordinates are independently sampled inside declared bounds. The fitted exponent uses only the smallest configured scales.

For query design, a candidate normalized grid is built independently of the unknown second-order coefficients. Every point contributes $D$ possible scalar component rows. We compare:
\begin{enumerate}
\item \emph{Greedy D-optimal}: sequentially maximize $b^\top A^{-1}b$ with a small ridge;
\item \emph{Pivoted QR}: select a rank-revealing minimal basis, then add high-norm rows;
\item \emph{Random}: uniform sampling without replacement;
\item \emph{Uniform}: equally spaced row indices, retained as a deliberately geometry-agnostic baseline.
\end{enumerate}
The response perturbation has a configured RMS and is regenerated from a recorded seed. The code records scalar-query count, number of unique physical locations, rank, $\sigma_{\min}$, condition number, and relative parameter error.

\subsection{Stability phase, scale-up, and response baselines}
The stability phase experiment uses three planar rays with configured minimum
angles and weights drawn from
\begin{equation*}
\Delta_{\min}\in\{3,5,8,12,20,30,45\}^{\circ},
\qquad
w_{\min}\in\{0.01,0.02,0.04,0.08,0.12,0.20\}.
\end{equation*}
Each of the 42 cells has 20 seeds, response-noise RMS $10^{-3}$, and twice as many scalar queries as unknowns. We retain the entire grid, including the high-error near-collision cells. The full grid characterizes the transition across angle and weight without imposing a universal separation threshold.

The scale-up crosses $(D,s)\in\{(2,4),(3,8),(5,8),(10,16),(20,16)\}$, ten seeds, and four first-order perturbation levels. Direction/weight RMS pairs are $(0,0)$, $(0.01,0.005)$, $(0.03,0.015)$, and $(0.06,0.03)$. Every case uses $\sigma=0.04$, $\lambda=2$, and twice the parameter count in scalar queries. Timing separates population-response evaluation, row design, and least-squares solution.

The response baseline uses four planar geometries, ten seeds, four base scales, population and KDE scores, and the same four perturbation levels. Three methods share fixed exact-geometry query rows: the matched two-scale response, naive subtraction of a tangent field, and an exact-$F_0$ diagnostic. The exact-$F_0$ diagnostic is an idealized reference that isolates the cost of estimating the nuisance field. Method ratios are paired before aggregation.

\subsection{KDE and estimated frontends}
A branch is selected according to its finite-radius mass and its radial coordinate is sampled by an inverse CDF. Empirical Gaussian KDE scores are evaluated exactly from the stored samples, with a per-query log-weight shift for numerical stability. The effective sample size
\begin{equation}
\ess(x)=\frac{(\sum_i W_i(x))^2}{\sum_i W_i(x)^2}
\end{equation}
is recorded at both scales. One sampled dataset is reused across the bandwidth sweep for a paired comparison.

The planar first-stage frontend queries a tangential score on an equally spaced shell, spectrally integrates the circular log density, divides by the positive Gaussian--ray Fourier multipliers, infers Toeplitz rank, recovers directions through a matrix pencil, and fits nonnegative weights. The end-to-end experiment reports known-count and blind-count variants separately. The second-order error is evaluated only when the recovered count matches, and a separate exact-first-order row isolates the second stage.

The nonplanar frontend treats the branch count as known and fits the leading tangent-ray score directly. Its variables are a normalized center offset, hyperspherical branch directions, and softmax-normalized weights. Multistart nonlinear least squares uses shell and interior score queries; the fitted center and first-order geometry then determine the second-stage query rows. Ground-truth directions and weights enter only the final permutation alignment and error calculation. The grid crosses $D\in\{3,4,5\}$, branch counts $\{3,4,6\}$, five seeds, population and KDE scores, and base scales $\{0.080,0.055,0.038\}$. KDE rows use \GeneralEtoESampleSize{} curve samples. Every case also includes an exact-first-order diagnostic evaluated at the same fitted center.

The crossed E4 diagnostic uses the same four geometries, five seeds, four scales, known/blind counts, population/KDE scores, and 32,768 samples. For every successful first stage, it independently crosses an estimated or exact first-order basis with rows selected from estimated or exact geometry. This separates coefficient-basis error from row-selection error. A failed blind-count case produces one explicit sentinel instead of four inverse rows; count accuracy includes that sentinel, whereas error summaries contain only successful inversions. Median relative error changes by a factor of \DiagnosticQueryRowRatio{} when the row policy is switched. In population scores, the estimated-basis/exact-basis ratio \DiagnosticPopulationEstimatedExactBasisRatio{} reflects finite-scale bias cancellation: the estimated tangent basis can absorb part of the remainder, so the comparison diagnoses bias absorption rather than geometric accuracy.

\subsection{Learned-score protocol and hardware scaling}
The network predicts the normalized score $\sigma s_\sigma$ from a normalized noisy location and a bounded-frequency sinusoidal encoding of $\log\sigma$. Clean samples are restricted to radius $1.25$ around the junction, and the noise scale is drawn log-uniformly across the configured interval. Training samples
\begin{equation}
\widetilde X=X/\sigma+Z,
\qquad Z\sim\mathcal N(0,I_D),
\end{equation}
then minimizes $\frac12\E\|f_\theta(\widetilde X,\sigma)+Z\|^2$. The residual MLP has width 128, four blocks, eight sinusoidal frequencies up to 8, and \LearnedModelParameters{} trainable parameters. Each job starts from 32,768 physical samples, of which \LearnedTrainingSampleMin{}--\LearnedTrainingSampleMax{} fall within the radius cutoff. AdamW uses batch size 1,024, initial learning rate $1.5\times10^{-3}$, weight decay $10^{-6}$, and gradient-norm clipping at 10. A cosine schedule decays to 5\% of the initial rate over 20,000 updates. Evaluation uses an exponential moving average with decay $0.999$ after a 1,000-update warmup at checkpoints $\{1000,2500,5000,7500,10000,15000,20000\}$.

All geometries, initializations, sample draws, minibatches, and noise draws use recorded integer seeds. The base configuration has one independent job per geometry--seed pair. The launcher assigns the \LearnedJobCount{} jobs round-robin across two NVIDIA GeForce RTX 3090 GPUs; each process sees one device as \texttt{cuda:0}. Median job time is \LearnedMedianJobSeconds{} seconds, aggregate GPU time is \LearnedGPUJobHours{} hours, and dual-GPU wall time is \LearnedElapsedSeconds{} seconds.

The scale-up retains the two harder four-ray geometries and three seeds, increases the model to width \LearnedScaleupWidth{} with \LearnedScaleupBlocks{} residual blocks, uses 65,536 physical samples, and trains through \LearnedScaleupSteps{} updates. The seven checkpoints are $\{1,5,10,20,30,40,50\}$k. The \LearnedScaleupJobs{} jobs were run concurrently on one RTX 3090; median per-job wall time was \LearnedScaleupMedianJobSeconds{} seconds. Independent population and empirical score evaluations diagnose the same controlled synthetic distribution. Jet-error checkpoint selection is reported post hoc; the nonmonotone parameter trajectory is the mechanism under study.

\subsection{Support Material}
The anonymous package contains \nolinkurl{analysis/configs/paper.yaml}, ten experiment entry points, and the shared numerical routines they import. The entry points cover the eleven studies above because \nolinkurl{run_learned_score.py} accepts both the base and scale-up configurations. Each script writes only CSV output to an ignored \nolinkurl{results/} directory. The package contains neither generated results nor code for manuscript figures or tables. Exact commands and the mapping from scripts to experiments appear in \nolinkurl{analysis/README.md}.

\subsection{Generated tables and additional figures}
The following tables report aggregates from the experiment CSVs and are included here for completeness.

\begin{table}[h]
\centering\small
\input{tables/population_summary.tex}
\caption{Population finite-scale fits from the active result profile.}
\end{table}

\begin{table}[h]
\centering\small
\input{tables/query_design_summary.tex}
\caption{Query-design summary at approximately twice the parameter dimension.}
\end{table}

\begin{table}[h]
\centering\small
\input{tables/kde_summary.tex}
\caption{Best bandwidth and error slopes on the retained twenty-seed grid.}
\end{table}

\begin{table}[h]
\centering\small
\input{tables/end_to_end_summary.tex}
\caption{Composition of the planar tangent frontend and the second-order inverse.}
\end{table}

\begin{table}[h]
\centering\small
\input{tables/general_end_to_end_summary.tex}
\caption{Known-count nonplanar composition from $D=3$ to $D=5$. Entries are medians over 45 cases per source and dimension; angle errors are in radians, jet errors are relative, and every displayed design is full rank.}
\label{tab:general-end-to-end}
\end{table}

\begin{table}[h]
\centering\small
\input{tables/learned_summary.tex}
\caption{Final-checkpoint learned-score diagnostics.}
\end{table}

\begin{table}[h]
\centering\small
\input{tables/stability_summary.tex}
\caption{Opposite corners of the angle--minimum-weight stability phase map.}
\end{table}

\begin{table}[h]
\centering\small
\input{tables/response_baseline_summary.tex}
\caption{Paired naive-subtraction error divided by matched-response error.}
\end{table}

\begin{table}[h]
\centering\scriptsize
\input{tables/scalability_summary.tex}
\caption{Scale-up errors and timings. Every displayed cell has ten seeds and full column rank.}
\end{table}

\begin{table}[h]
\centering\scriptsize
\resizebox{\linewidth}{!}{\input{tables/end_to_end_diagnostics_summary.tex}}
\caption{Crossed E4 diagnostic on successful inverse rows. ``Basis'' and ``Rows'' vary independently.}
\end{table}

\begin{table}[h]
\centering\small
\input{tables/learned_scaleup_summary.tex}
\caption{Width-256, six-block hard-case learned-score follow-up. Values are medians over six jobs.}
\end{table}

\begin{figure}[h]
\centering
\begin{subfigure}[t]{0.38\linewidth}
  \centering
  \includegraphics[width=\linewidth]{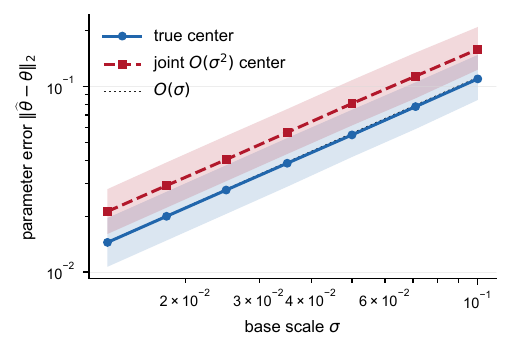}
  \caption{Predicted population order.}
\end{subfigure}\hfill
\begin{subfigure}[t]{0.60\linewidth}
  \centering
  \includegraphics[width=\linewidth]{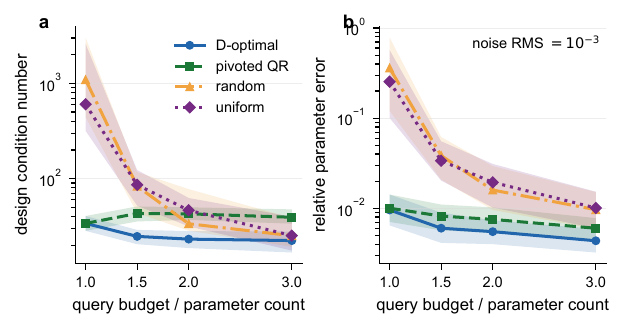}
  \caption{Conditioning controls recovery.}
\end{subfigure}
\caption{Two theorem-facing checks. Lines show medians and bands show interquartile ranges. Left: exact- and joint-center errors follow the $O(\sigma)$ reference. Right: additional, well-selected scalar queries stabilize noisy recovery.}
\label{fig:paper-main}
\end{figure}

\begin{figure}[h]
\centering
\includegraphics[width=0.94\linewidth]{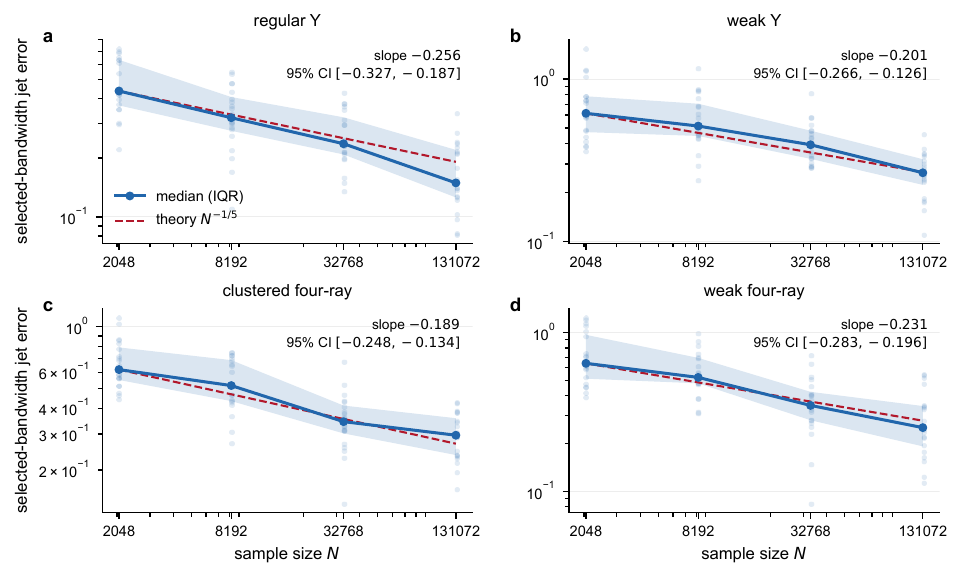}
\caption{KDE finite-sample rate by geometry. Faint points are individual seeds at the median-selected bandwidth; solid curves and bands show the seed median and interquartile range. Red references have slope $-1/5$. Insets report 5,000-resample seed-bootstrap percentile intervals.}
\label{fig:kde-grid}
\end{figure}

\begin{figure}[h]
\centering
\includegraphics[width=0.88\linewidth]{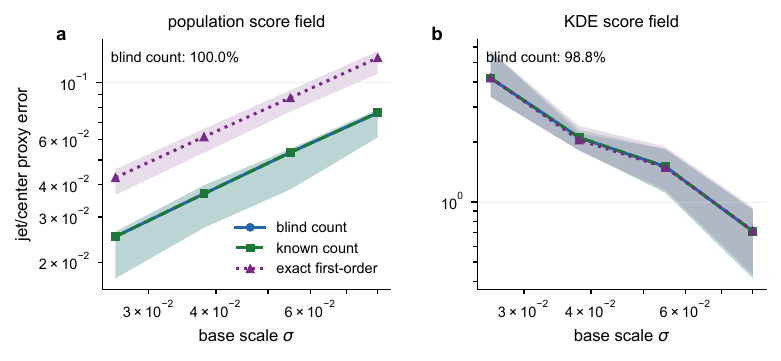}
\caption{First-order frontend followed by second-order recovery. Bands are interquartile ranges over correct-count rows. Population error decreases with scale; at fixed sample size, KDE variance reverses the trend. Blind-count accuracy is computed on all rows.}
\label{fig:end-to-end-grid}
\end{figure}

\begin{figure}[h]
\centering
\includegraphics[width=0.98\linewidth]{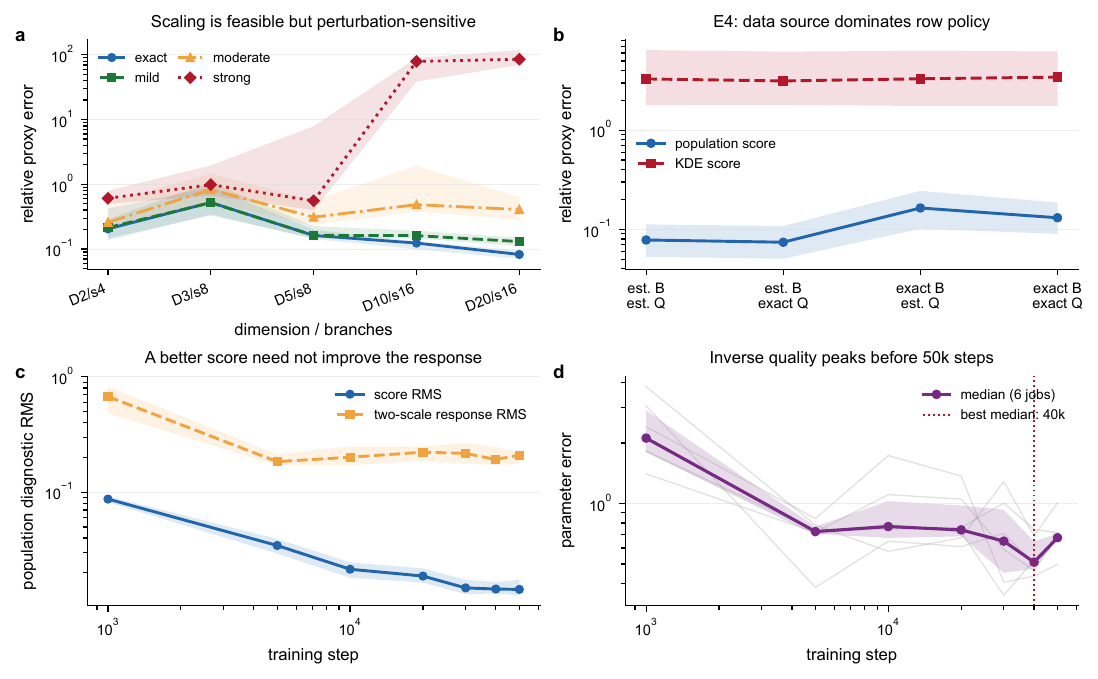}
\caption{Supplemental diagnostics. (a) Dimension/branch scale-up under exact through strong first-order perturbations. (b) E4 with first-order basis (B) and query-row geometry (Q) crossed. (c--d) Width-256 learned-score follow-up through 50k updates; bands are interquartile ranges, and thin curves in (d) are all six jobs.}
\label{fig:revision-diagnostics}
\end{figure}

\begin{figure}[h]
\centering
\includegraphics[width=0.72\linewidth]{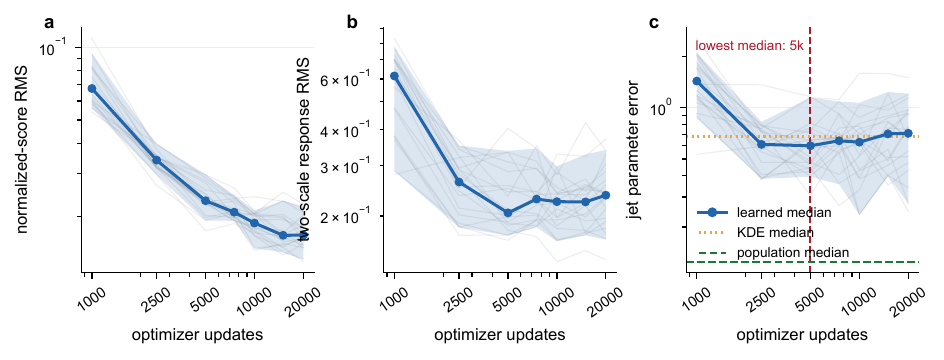}
\caption{Learned-score diagnostics across seven checkpoints. Thin curves show all \LearnedJobCount{} jobs; thick curves and bands show the median and 5th--95th percentiles. Median jet error reaches its minimum at \LearnedBestMedianStep{} updates before score RMS reaches its lowest value; dotted and dashed lines are matched KDE and population baselines.}
\label{fig:learned-grid}
\end{figure}

%% file: tables/population_summary.tex
\begin{tabular}{lccccc}
\toprule
$D$ & center & median slope & IQR & median $R^2$ & instances \\
\midrule
2 & biased & 1.005 & [0.993, 1.021] & 1.0000 & 30 \\
2 & true & 0.992 & [0.978, 1.003] & 1.0000 & 30 \\
3 & biased & 1.003 & [0.989, 1.009] & 1.0000 & 30 \\
3 & true & 0.999 & [0.992, 1.010] & 1.0000 & 30 \\
5 & biased & 0.996 & [0.988, 1.001] & 1.0000 & 30 \\
5 & true & 0.999 & [0.992, 1.002] & 1.0000 & 30 \\
\bottomrule
\end{tabular}

%% file: tables/query_design_summary.tex
\begin{tabular}{lcccc}
\toprule
design & median $\sigma_{\min}$ & median cond. & median relative error & 95th pct. \\
\midrule
D-optimal & 0.468 & 23.2 & 0.00552 & 0.011 \\
pivoted QR & 0.321 & 42.7 & 0.00754 & 0.0151 \\
random & 0.138 & 33.6 & 0.016 & 0.0629 \\
uniform & 0.109 & 46.6 & 0.0195 & 0.0714 \\
\bottomrule
\end{tabular}

%% file: tables/kde_summary.tex
\begin{tabular}{lcccc}
\toprule
geometry & error slope & seed-bootstrap 95\% CI & theory & bandwidth slope \\
\midrule
clustered four-ray & -0.189 & [-0.248, -0.134] & -0.200 & -0.184 \\
regular Y & -0.256 & [-0.327, -0.187] & -0.200 & -0.163 \\
weak four-ray & -0.231 & [-0.283, -0.196] & -0.200 & -0.200 \\
weak Y & -0.201 & [-0.266, -0.126] & -0.200 & -0.159 \\
\bottomrule
\end{tabular}

%% file: tables/end_to_end_summary.tex
\begin{tabular}{lccccc}
\toprule
field & tangent input & count rate & median jet error & 95th pct. & median angle \\
\midrule
kde & blind count & 98.8\% & 1.77 & 5.78 & 0.0216 \\
kde & known count & 100.0\% & 1.79 & 5.76 & 0.0214 \\
kde & exact first-order & 100.0\% & 1.8 & 5.77 & 0 \\
population & blind count & 100.0\% & 0.0396 & 0.0892 & 0.0177 \\
population & known count & 100.0\% & 0.0396 & 0.0892 & 0.0177 \\
population & exact first-order & 100.0\% & 0.0672 & 0.15 & 0 \\
\bottomrule
\end{tabular}

%% file: tables/general_end_to_end_summary.tex
\begin{tabular}{@{}ccccc@{}}
\toprule
$D$ & Pop. angle & Pop. jet error & KDE angle & KDE jet error \\
\midrule
3 & 0.0212 & 0.142 & 0.0235 & 0.975 \\
4 & 0.0167 & 0.124 & 0.0206 & 0.967 \\
5 & 0.0209 & 0.125 & 0.0219 & 0.927 \\
\bottomrule
\end{tabular}

%% file: tables/learned_summary.tex
\begin{tabular}{lccccc}
\toprule
geometry & seeds & score RMS & scale-diff. RMS & jet error & KDE jet error \\
\midrule
clustered four-ray & 5 & 0.0149 & 0.231 & 0.972 & 0.861 \\
regular Y & 5 & 0.0164 & 0.215 & 0.34 & 0.277 \\
weak four-ray & 5 & 0.0169 & 0.238 & 0.746 & 0.599 \\
weak Y & 5 & 0.0203 & 0.303 & 0.761 & 0.677 \\
\bottomrule
\end{tabular}

%% file: tables/stability_summary.tex
\begin{tabular}{lccccc}
\toprule
Regime & Angle & Min. weight & Median $\sigma_{\min}$ & Median cond. & Median rel. error \\
\midrule
near collision & 3 & 0.01 & 0.0018 & 5.61e+04 & 0.698 \\
well separated & 45 & 0.2 & 0.553 & 24.6 & 0.00379 \\
\bottomrule
\end{tabular}

%% file: tables/response_baseline_summary.tex
\begin{tabular}{lcccc}
\toprule
Score source & Perturbation & Median naive/matched & IQR & $n$ \\
\midrule
population & exact & 0.334 & [0.332, 0.337] & 160 \\
population & mild & 7.75 & [3.5, 15.7] & 160 \\
population & moderate & 26.6 & [10.2, 52.4] & 160 \\
population & strong & 49.4 & [21.4, 115] & 160 \\
KDE & exact & 0.863 & [0.652, 1.24] & 160 \\
KDE & mild & 0.893 & [0.676, 1.32] & 160 \\
KDE & moderate & 1.21 & [0.916, 1.78] & 160 \\
KDE & strong & 1.96 & [1.32, 3.09] & 160 \\
\bottomrule
\end{tabular}

%% file: tables/scalability_summary.tex
\begin{tabular}{lcccc}
\toprule
Case & Perturbation & Median rel. error & Median cond. & Response time (s) \\
\midrule
D2/s4 & exact & 0.205 & 40.2 & 0.011 \\
D2/s4 & mild & 0.215 & 40.1 & 0.011 \\
D2/s4 & moderate & 0.259 & 38.2 & 0.011 \\
D2/s4 & strong & 0.609 & 38.5 & 0.011 \\
D3/s8 & exact & 0.524 & 169 & 0.022 \\
D3/s8 & mild & 0.52 & 143 & 0.022 \\
D3/s8 & moderate & 0.831 & 174 & 0.022 \\
D3/s8 & strong & 0.989 & 307 & 0.022 \\
D5/s8 & exact & 0.162 & 44 & 0.030 \\
D5/s8 & mild & 0.163 & 49.7 & 0.030 \\
D5/s8 & moderate & 0.311 & 60.3 & 0.030 \\
D5/s8 & strong & 0.556 & 117 & 0.030 \\
D10/s16 & exact & 0.124 & 99.8 & 0.096 \\
D10/s16 & mild & 0.161 & 104 & 0.096 \\
D10/s16 & moderate & 0.485 & 175 & 0.096 \\
D10/s16 & strong & 79.2 & 8.55e+03 & 0.096 \\
D20/s16 & exact & 0.0825 & 128 & 0.121 \\
D20/s16 & mild & 0.131 & 137 & 0.121 \\
D20/s16 & moderate & 0.408 & 162 & 0.121 \\
D20/s16 & strong & 85.6 & 2.43e+04 & 0.121 \\
\bottomrule
\end{tabular}

%% file: tables/end_to_end_diagnostics_summary.tex
\begin{tabular}{lcccccc}
\toprule
Source & Basis & Rows & Projected error & Relative proxy error & Remainder & Cond. \\
\midrule
population & estimated & estimated & 0.039 & 0.0781 & 0.0219 & 20.7 \\
population & estimated & exact & 0.0376 & 0.0741 & 0.0219 & 20.3 \\
population & exact & estimated & 0.0785 & 0.164 & 0.0387 & 20.8 \\
population & exact & exact & 0.0675 & 0.131 & 0.0375 & 20.4 \\
kde & estimated & estimated & 1.68 & 3.29 & 0.852 & 19.8 \\
kde & estimated & exact & 1.67 & 3.14 & 0.844 & 19.7 \\
kde & exact & estimated & 1.63 & 3.3 & 0.851 & 19.5 \\
kde & exact & exact & 1.66 & 3.44 & 0.833 & 19.5 \\
\bottomrule
\end{tabular}

%% file: tables/learned_scaleup_summary.tex
\begin{tabular}{lcccc}
\toprule
Step & Score RMS & Response RMS & Parameter error & Parameter-error IQR \\
\midrule
1k & 0.0872 & 0.675 & 2.11 & [1.81, 2.9] \\
5k & 0.0345 & 0.184 & 0.725 & [0.705, 0.768] \\
10k & 0.0215 & 0.201 & 0.768 & [0.674, 1.03] \\
20k & 0.0187 & 0.222 & 0.739 & [0.686, 0.977] \\
30k & 0.0148 & 0.217 & 0.65 & [0.452, 0.931] \\
40k & 0.0145 & 0.192 & 0.511 & [0.473, 0.648] \\
50k & 0.0143 & 0.21 & 0.676 & [0.665, 0.708] \\
\bottomrule
\end{tabular}

%% file: refs.bib
@misc{weightedtangent2026,
  title        = {Recovering Weighted Tangent Geometry from a Single-Scale Score Field},
  author       = {Zhao, Ziqi and Ni, Qingjian},
  year         = {2026},
  howpublished = {arXiv:2608.22334}
}

@misc{brosse2026boundary,
  title        = {Boundary-Layer Asymptotics for Gaussian-Smoothed Singular Measures},
  author       = {Brosse, Nicolas and Dalalyan, Arnak S.},
  year         = {2026},
  howpublished = {arXiv:2607.04514}
}

@article{hyvarinen2005score,
  title   = {Estimation of Non-Normalized Statistical Models by Score Matching},
  author  = {Hyv{\"a}rinen, Aapo},
  journal = {Journal of Machine Learning Research},
  volume  = {6},
  number  = {24},
  pages   = {695--709},
  year    = {2005}
}

@inproceedings{song2019generative,
  title     = {Generative Modeling by Estimating Gradients of the Data Distribution},
  author    = {Song, Yang and Ermon, Stefano},
  booktitle = {Advances in Neural Information Processing Systems},
  volume    = {32},
  year      = {2019}
}

@inproceedings{song2021score,
  title     = {Score-Based Generative Modeling through Stochastic Differential Equations},
  author    = {Song, Yang and Sohl-Dickstein, Jascha and Kingma, Diederik P. and Kumar, Abhishek and Ermon, Stefano and Poole, Ben},
  booktitle = {International Conference on Learning Representations},
  year      = {2021}
}

@article{vincent2011connection,
  title   = {A Connection Between Score Matching and Denoising Autoencoders},
  author  = {Vincent, Pascal},
  journal = {Neural Computation},
  volume  = {23},
  number  = {7},
  pages   = {1661--1674},
  year    = {2011},
  doi     = {10.1162/NECO_a_00142}
}

@article{alain2014regularized,
  title   = {What Regularized Auto-Encoders Learn from the Data-Generating Distribution},
  author  = {Alain, Guillaume and Bengio, Yoshua},
  journal = {Journal of Machine Learning Research},
  volume  = {15},
  number  = {110},
  pages   = {3743--3773},
  year    = {2014}
}

@inproceedings{stanczuk2024dimension,
  title     = {Diffusion Models Encode the Intrinsic Dimension of Data Manifolds},
  author    = {Stanczuk, Jan Pawel and Batzolis, Georgios and Deveney, Teo and Sch{\"o}nlieb, Carola-Bibiane},
  booktitle = {International Conference on Machine Learning},
  year      = {2024}
}

@inproceedings{ventura2025manifolds,
  title     = {Manifolds, Random Matrices and Spectral Gaps: The Geometric Phases of Generative Diffusion},
  author    = {Ventura, Enrico and Achilli, Beatrice and Silvestri, Gianluigi and Lucibello, Carlo and Ambrogioni, Luca},
  booktitle = {International Conference on Learning Representations},
  year      = {2025}
}

@inproceedings{diepeveen2025pullback,
  title     = {Score-Based Pullback Riemannian Geometry: Extracting the Data Manifold Geometry Using Anisotropic Flows},
  author    = {Diepeveen, Willem and Batzolis, Georgios and Shumaylov, Zakhar and Sch{\"o}nlieb, Carola-Bibiane},
  booktitle = {International Conference on Machine Learning},
  year      = {2025}
}

@inproceedings{kharitenko2026landing,
  title     = {Landing with the Score: Riemannian Optimization through Denoising},
  author    = {Kharitenko, Andrey and Shen, Zebang and De Santi, Riccardo and He, Niao and D{\"o}rfler, Florian},
  booktitle = {International Conference on Learning Representations},
  year      = {2026}
}

@inproceedings{li2026rates,
  title     = {When Scores Learn Geometry: Rate Separations under the Manifold Hypothesis},
  author    = {Li, Xiang and Shen, Zebang and Hsieh, Ya-Ping and He, Niao},
  booktitle = {International Conference on Learning Representations},
  year      = {2026}
}

@misc{zhang2026manifold,
  title        = {Diffusion Model for Manifold Data: Score Decomposition, Curvature, and Statistical Complexity},
  author       = {Zhang, Zixuan and Huang, Kaixuan and Zhao, Tuo and Wang, Mengdi and Chen, Minshuo},
  year         = {2026},
  howpublished = {arXiv:2603.20645}
}

@misc{rawal2026rao,
  title        = {Rao-Blackwellized Score Matching on Manifolds},
  author       = {Rawal, Divit},
  year         = {2026},
  howpublished = {arXiv:2605.25567}
}

@article{federer1959curvature,
  title   = {Curvature Measures},
  author  = {Federer, Herbert},
  journal = {Transactions of the American Mathematical Society},
  volume  = {93},
  number  = {3},
  pages   = {418--491},
  year    = {1959},
  doi     = {10.1090/S0002-9947-1959-0110078-1}
}

@article{preiss1987geometry,
  title   = {Geometry of Measures in $\mathbb{R}^n$: Distribution, Rectifiability, and Densities},
  author  = {Preiss, David},
  journal = {Annals of Mathematics},
  volume  = {125},
  number  = {3},
  pages   = {537--643},
  year    = {1987}
}

@article{niyogi2008homology,
  title   = {Finding the Homology of Submanifolds with High Confidence from Random Samples},
  author  = {Niyogi, Partha and Smale, Stephen and Weinberger, Shmuel},
  journal = {Discrete \& Computational Geometry},
  volume  = {39},
  number  = {1--3},
  pages   = {419--441},
  year    = {2008},
  doi     = {10.1007/s00454-008-9053-2}
}

@article{ozertem2011principal,
  title   = {Locally Defined Principal Curves and Surfaces},
  author  = {Ozertem, Umut and Erdogmus, Deniz},
  journal = {Journal of Machine Learning Research},
  volume  = {12},
  number  = {34},
  pages   = {1249--1286},
  year    = {2011}
}

@article{genovese2012minimax,
  title   = {Minimax Manifold Estimation},
  author  = {Genovese, Christopher R. and Perone-Pacifico, Marco and Verdinelli, Isabella and Wasserman, Larry},
  journal = {Journal of Machine Learning Research},
  volume  = {13},
  number  = {43},
  pages   = {1263--1291},
  year    = {2012}
}

@article{aamari2019nonasymptotic,
  title   = {Nonasymptotic Rates for Manifold, Tangent Space and Curvature Estimation},
  author  = {Aamari, Eddie and Levrard, Cl{\'e}ment},
  journal = {The Annals of Statistics},
  volume  = {47},
  number  = {1},
  pages   = {177--204},
  year    = {2019},
  doi     = {10.1214/18-AOS1685}
}

@book{munkres2000topology,
  title     = {Topology},
  author    = {Munkres, James R.},
  edition   = {2},
  publisher = {Prentice Hall},
  year      = {2000}
}

@article{pinelis1994martingales,
  title   = {Optimum Bounds for the Distributions of Martingales in Banach Spaces},
  author  = {Pinelis, Iosif},
  journal = {The Annals of Probability},
  volume  = {22},
  number  = {4},
  pages   = {1679--1706},
  year    = {1994}
}

@article{kunis2016prony,
  title   = {A Multivariate Generalization of Prony's Method},
  author  = {Kunis, Stefan and Peter, Thomas and R{\"o}mer, Tim and von der Ohe, Ulrich},
  journal = {Linear Algebra and its Applications},
  volume  = {490},
  pages   = {31--47},
  year    = {2016}
}

@article{kunis2019sphere,
  title   = {Prony's Method on the Sphere},
  author  = {Kunis, Stefan and M{\"o}ller, H. Michael and von der Ohe, Ulrich},
  journal = {SMAI Journal of Computational Mathematics},
  volume  = {S5},
  pages   = {87--97},
  year    = {2019}
}

@article{candes2014superresolution,
  title   = {Towards a Mathematical Theory of Super-Resolution},
  author  = {Cand{\`e}s, Emmanuel J. and Fernandez-Granda, Carlos},
  journal = {Communications on Pure and Applied Mathematics},
  volume  = {67},
  number  = {6},
  pages   = {906--956},
  year    = {2014}
}
